\documentclass[]{onecat}

\usepackage[utf8]{inputenc}
\usepackage{caption}
\usepackage{hyperref}
\usepackage{cleveref}
\usepackage{verbatim}
\usepackage{amsmath}
\usepackage{amssymb}
\usepackage{amsfonts}
\usepackage{amsthm}
\usepackage{mathtools}
\usepackage{graphicx}
\usepackage{floatrow}
\usepackage{subcaption}
\usepackage{listings}
\usepackage{algorithm}
\usepackage{makecell}
\usepackage{multirow}
\usepackage{tabularx}
\usepackage{array}
\usepackage{textcomp}
\usepackage{enumitem}
\usepackage{wrapfig}
\usepackage{tikz}
\usepackage[toc,page,header]{appendix}
\usepackage{pifont}
\usepackage{minitoc}
\usepackage{marvosym}

\usetikzlibrary{arrows.meta,positioning,fit,calc}

\definecolor{PromptFrame}{HTML}{C99700}
\definecolor{PromptTitle}{HTML}{D9A600}
\definecolor{PromptBack}{HTML}{FFF4BF}
\definecolor{SpecFrame}{HTML}{4F7CAC}
\definecolor{SpecTitle}{HTML}{D9EAF7}
\definecolor{SpecBack}{HTML}{F5FAFF}
\definecolor{SchemaFrame}{HTML}{5F6B7A}
\definecolor{SchemaBack}{HTML}{F7F8FA}
\definecolor{CaseFrame}{HTML}{6B7280}
\definecolor{CaseBack}{HTML}{F8FAFC}
\definecolor{CaseTitle}{HTML}{E5E7EB}
\definecolor{GoodBack}{HTML}{ECFDF5}
\definecolor{BadBack}{HTML}{FEF2F2}
\definecolor{NeutralBack}{HTML}{F9FAFB}

\newcommand{\modelname}{\textsc{TraceLift}}
\newcommand{\method}{\modelname}
\newcommand{\datasetname}{\textsc{TraceLift-Groups}}

\newcommand{\best}[1]{{\bfseries #1}}
\newcommand{\dflab}{D\textsuperscript{4} Lab}

\newcommand{\Rexec}{R_{\mathrm{exec}}}
\newcommand{\urm}{u_{\mathrm{exec}}}
\newcommand{\srm}{s_{\mathrm{rm}}}
\newcommand{\clip}{\operatorname{clip}}

\newcommand{\Aout}{A}
\newcommand{\Dist}{\mathcal{D}}
\newcommand{\Expect}{\mathbb{E}}
\newcommand{\Prob}{\mathbb{P}}
\newcommand{\Var}{\operatorname{Var}}

\newcommand{\qtr}{q_{\mathrm{tr}}}
\newcommand{\qev}{q_{\mathrm{eval}}}
\newcommand{\utr}{u_{\mathrm{tr}}}
\newcommand{\mrm}{m_{\phi}}
\newcommand{\rempty}{\varnothing}
\newcommand{\badtag}[1]{\colorbox{BadBack}{\strut\textsf{#1}}}

\newcolumntype{Y}{>{\raggedright\arraybackslash}X}
\theoremstyle{definition}

\newtheorem{appassumption}{Assumption}
\newtheorem{appdefinition}{Definition}
\newtheorem{applemma}{Lemma}
\newtheorem{appproposition}{Proposition}
\newtheorem{appcorollary}{Corollary}
\newtheorem{appremark}{Remark}

\lstdefinestyle{promptstyle}{
  basicstyle=\ttfamily\scriptsize,
  breaklines=true,
  breakatwhitespace=false,
  columns=fullflexible,
  keepspaces=true,
  showstringspaces=false,
  tabsize=2,
  upquote=true
}

\lstdefinelanguage{json}{
  basicstyle=\ttfamily\small,
  breaklines=true,
  morecomment=[l]{//},
  morestring=[b]",
  stringstyle=\color{red!70!black},
  commentstyle=\color{gray!70!white},
  literate=
    *{0}{{{\color{gray}{0}}}}1
     {1}{{{\color{gray}{1}}}}1
     {:}{{{\color{black}{:}}}}1
     {,}{{{\color{black}{,}}}}1
     {\{}{{{\color{blue}{\{}}}}1
     {\}}{{{\color{blue}{\}}}}}1
     {[}{{{\color{blue}{[}}}}1
     {]}{{{\color{blue}{]}}}}1
}

\tcbuselibrary{listingsutf8}

\newtcolorbox{casebox}[2][]{%
  enhanced,
  breakable,
  sharp corners,
  colback=CaseBack,
  colframe=CaseFrame,
  colbacktitle=CaseTitle,
  coltitle=black,
  fonttitle=\bfseries,
  title={#2},
  left=1.2mm,
  right=1.2mm,
  top=1mm,
  bottom=1mm,
  #1
}

\newtcblisting{promptlisting}[2][]{%
  enhanced,
  breakable,
  sharp corners,
  colback=PromptBack,
  colframe=PromptFrame,
  colbacktitle=PromptTitle,
  coltitle=black,
  fonttitle=\bfseries,
  title={#2},
  listing only,
  listing options={style=promptstyle},
  left=1mm,
  right=1mm,
  top=1mm,
  bottom=1mm,
  #1
}

\newtcolorbox{specbox}[2][]{%
  enhanced,
  breakable,
  sharp corners,
  colback=SpecBack,
  colframe=SpecFrame,
  colbacktitle=SpecTitle,
  coltitle=black,
  fonttitle=\bfseries,
  title={#2},
  left=1.2mm,
  right=1.2mm,
  top=1mm,
  bottom=1mm,
  #1
}

\newtcblisting{schemalisting}[2][]{%
  enhanced,
  breakable,
  sharp corners,
  colback=SchemaBack,
  colframe=SchemaFrame,
  colbacktitle=SchemaFrame!20,
  coltitle=black,
  fonttitle=\bfseries,
  title={#2},
  listing only,
  listing options={style=promptstyle},
  left=1mm,
  right=1mm,
  top=1mm,
  bottom=1mm,
  #1
}

\title{\raisebox{-0.2cm}{\includegraphics[width=1.05cm]{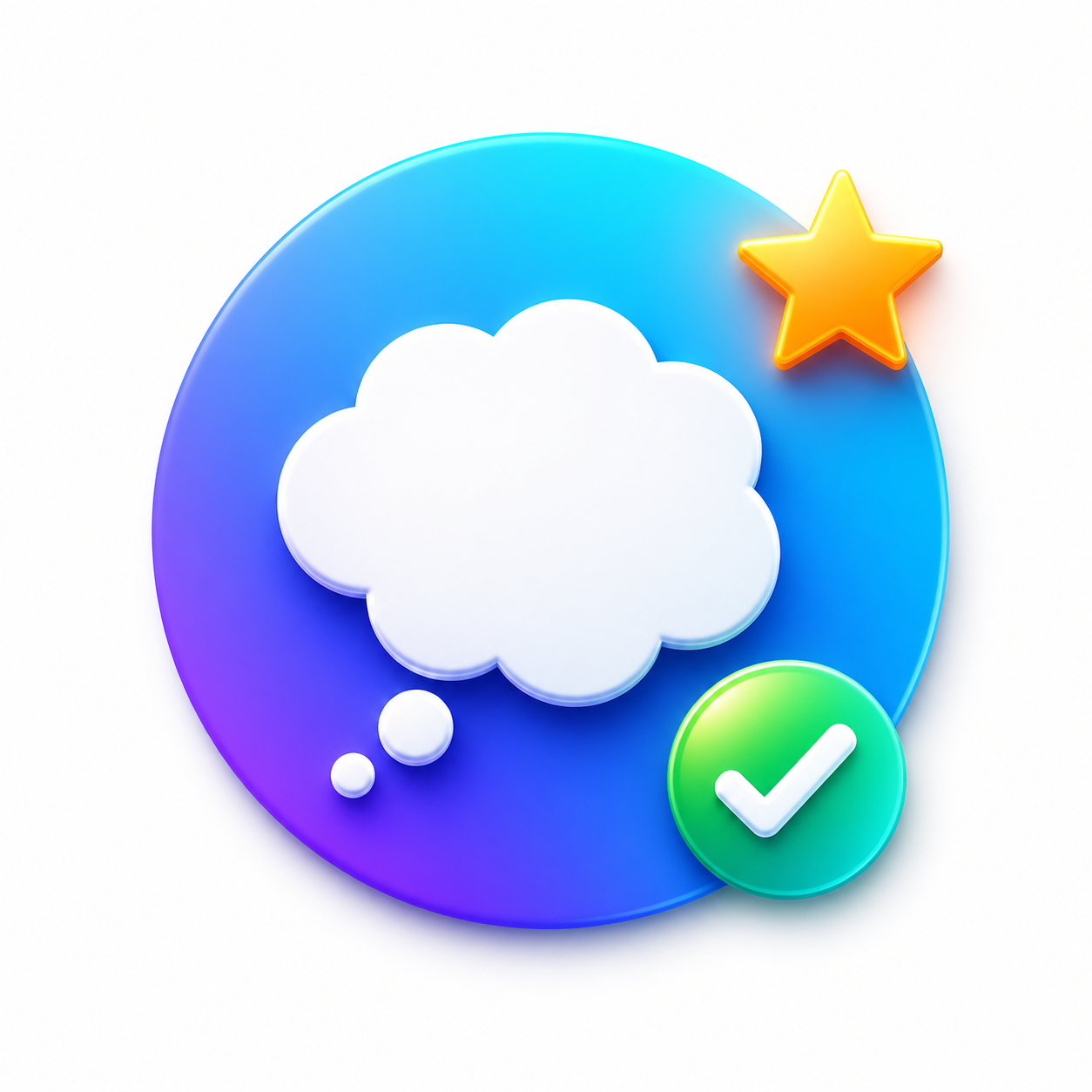}} Correct Is Not Enough: Training Reasoning Planners with Executor-Grounded Rewards}

\author{%
\parbox{\textwidth}{\centering
Tianyang Han$^{1*}$, Hengyu Shi$^{2*}$, Junjie Hu$^{2*}$ \\ [0.5em]
Xu Yang$^{1}$, Zhiling Wang$^{2\ddagger}$, Junhao Su$^{2\ddagger\dagger}$\\
}
}

\affiliation{%
\parbox{\textwidth}{\centering\small
$^1$\dflab, \quad $^2$Independent Researcher
}}

\contribution[*]{Equal contribution}
\contribution[\ddagger]{Corresponding author}
\contribution[\dagger]{Project leader}

\abstract{
Reinforcement learning with verifiable rewards has become a common way to improve explicit reasoning in large language models, but final-answer correctness alone does not reveal whether the reasoning trace is faithful, reliable, or useful to the model that consumes it. This outcome-only signal can reinforce traces that are right for the wrong reasons, overstate reasoning gains by rewarding shortcuts, and propagate flawed intermediate states in multi-step systems.
To this end, we propose \modelname, a planner-executor training framework that treats reasoning as a consumable intermediate artifact. During planner training, the planner emits tagged reasoning. A frozen executor turns this reasoning into the final artifact for verifier feedback, while an executor-grounded reward shapes the intermediate trace. This reward multiplies a rubric-based Reasoning Reward Model (RM) score by measured uplift on the same frozen executor, crediting traces that are both high-quality and useful. To make reasoning quality directly learnable, we introduce \datasetname{}, a rubric-annotated reason-only dataset built from math and code seed problems. Each example is a same-problem group containing a high-quality reference trace and multiple plausible flawed traces with localized perturbations that reduce reasoning quality or solution support while preserving task relevance.
Extensive experiments on code and math benchmarks show that this executor-grounded reasoning reward improves the two-stage planner-executor system over execution-only training, suggesting that reasoning supervision should evaluate not only whether a trace looks good, but also whether it helps the model that consumes it.
Our code is available at: \url{https://github.com/MasaiahHan/TraceLift}
}
\date{\today}

\begin{document}

\maketitle

\section{Introduction}
\label{sec:introduction}

Large language models increasingly use explicit reasoning trajectories especially during post-training with verifiable feedback. Chain-of-thought, bootstrapped rationales, program-aided reasoning, and recent reinforcement learning methods encourage models to produce intermediate text before final answers or programs~\citep{wei2022chain,kojima2022large,wang2022self,yao2023tree,zelikman2022star,gsmhard,chen2023program,shao2024deepseekmath,deepseek2025r1,Han2025BeyondWA}. These trajectories are not merely explain the plan, constraints, intermediate derivations, and implementation choices but help to shape the final artifact. Reinforcement learning with verifiable feedback provides a natural way to improve reasoning-based systems.~\citep{shao2024deepseekmath,deepseek2025r1,le2022coderl,liu2023rltf,shojaee2023execution, su2025failure,pi2024strengthening}

As reasoning trajectories become part of the computation, reward signals must evaluate not only final correctness but also the quality and downstream effect of the reasoning process itself. Existing reward signals indicate whether a final answer or program succeeds, but they can hide flawed intermediate reasoning behind correct artifacts. This creates a credit-assignment failure in verifiable reinforcement learning: traces that are right for the wrong reasons, rely on shortcuts, or mask early mistakes with later fluent text can still be reinforced as successful behavior. Process supervision and automated process verifiers provide more direct reasoning feedback, yet language-model-based judges can inherit reasoning errors, calibration failures, or hallucinated assessments~\citep{uesato2022solving,gsmhard,wang2024math,setlur2024rewarding,han2024instinctive,pi2024mllm}. These limitations leave underexplored how to train reasoning trajectories that are both internally reliable and useful to the model that consumes them.

Owing to the modular role that reasoning trajectories already play in reasoning-based systems, they can be viewed as an interface between a trainable planner and a fixed executor. Under this view, the planner should not be rewarded merely for producing coherent-looking text. It should be rewarded for producing an intermediate trajectory that is reliable as a reasoning process and useful to the executor that consumes it. This perspective changes both the training and evaluation targets. Training should assign credit to trajectories according to their intrinsic quality and their measured effect on the executor. Evaluation should hold the executor fixed so that improvements reflect better executor-consumable reasoning rather than a stronger final-answer generator.

To address these limitations, we propose \modelname{} as a planner-executor framework that treats reasoning as a consumable intermediate artifact in a controlled planner-executor setting. We introduce a trainable reasoning planner policy that produces an intermediate trajectory consumed by a fixed executor rather than a final artifact for evaluation. The frozen executor uses the problem and trajectory to generate the final artifact, which is then scored by a verifier using task-specific correctness checks. This separation grounds reasoning supervision in the executor that will actually use the trajectory. During planner training, a rubric-based Reasoning Reward Model (RM) scores the trajectory itself and this score is weighted by measured executor uplift over a no-reasoning baseline. The resulting reward gives credit to reasoning only when it is both high quality and useful to the fixed executor. At evaluation time, the Reason RM and measured-uplift computation are removed. Improvements therefore reflect whether the trained planner better guides the same frozen executor under the same two-stage evaluation chain.

Specifically, to overcome the data bottleneck in learning reason-only quality signals, we introduce \datasetname{} as a rubric-annotated reason-only dataset. \datasetname{} contains 6,000 reasoning groups sampled from GSM8K train and OpenCodeReasoning. Each group anchors a high-quality reference trajectory and multiple plausible flawed trajectories to the same problem. These flawed trajectories are produced by targeted perturbations that weaken reasoning quality or solution support while preserving task relevance. An agentic rubric annotation pipeline then assigns multi-dimensional reasoning-quality scores. This grouped reason-only format trains the Reason RM to score a problem-reasoning pair by trajectory quality rather than final-solution correctness.

Our contributions are as follows.
\begin{itemize}
  \item We introduce \modelname{}, an executor-grounded training framework for reasoning planners. It rewards trajectories using both rubric-based reasoning quality and measured uplift on the same frozen executor, making the training signal aligned with the model that will consume the trajectory at test time.
  \item We introduce \datasetname{}, a rubric-annotated reason-only dataset with 6,000 reasoning groups sampled from GSM8K train and OpenCodeReasoning. Each group pairs a reference trajectory with targeted flawed trajectories and multi-dimensional rubric annotations for reasoning-trajectory quality rather than final-artifact correctness.
  \item Extensive experiments on code and math benchmarks demonstrate that \modelname{} consistently improves fixed two-stage planner-executor systems over execution-only reinforcement learning under the same evaluation chain, highlighting the advantage of rewarding reasoning by both trajectory quality and downstream executor utility.
\end{itemize}

\section{Related Work}
\label{sec:related}

\paragraph{ Large Language Models for Reasoning}
Chain-of-thought prompting shows that explicitly decoding intermediate reasoning can improve multi-step problem solving~\citep{wei2022chain,kojima2022large,wang2022self}. Search and bootstrapping methods further scale reasoning by sampling, refining, or selecting intermediate processes~\citep{yao2023tree,zelikman2022star}. Agentic methods such as ReAct and Reflexion use language-based reasoning to choose subsequent actions or incorporate feedback~\citep{yao2023react,shinn2023reflexion}. Program-aided methods connect language reasoning with external computation~\citep{gsmhard,chen2023program}. Recent reasoning-oriented models such as OpenAI o1, Qwen, and DeepSeek-R1 further show that explicit reasoning has become central to strong math and code performance~\citep{openai2024o1,yang2024qwen2,yang2025qwen3,deepseek2025r1}. These works establish reasoning as more than post-hoc explanation. Our work asks a complementary training question. We study how to reward a reasoning planner so that its intermediate reasoning is not only plausible but also useful to a separate frozen executor under a fixed evaluation protocol.

\paragraph{Agentic reinforcement learning and trajectory-level rewards.}
Agentic language-model systems increasingly treat problem solving as a trajectory of reasoning, actions, and feedback. ReAct and Reflexion show that language-based reasoning can guide actions and incorporate feedback~\citep{yao2023react,shinn2023reflexion,su2026manpp}. RLHF and preference-optimization methods train policies from learned preferences~\citep{christiano2017deep,ziegler2019fine,ouyang2022training,rafailov2024direct,su2024momentum}. Verifiable reinforcement learning and execution-based code methods further use answer correctness, compiler signals, or unit tests as rewards~\citep{shao2024deepseekmath,deepseek2025r1,le2022coderl,liu2023rltf,shojaee2023execution,li2022competition}. Process supervision and reward modeling move beyond final outcomes by assigning credit to intermediate steps or complete solution trajectories~\citep{gsm8k,uesato2022solving,math500,wang2024math,setlur2024rewarding}. These works motivate trajectory-level evaluation rather than outcome-only scoring. However, they do not directly evaluate whether the reasoning trajectory is accurate enough to help the downstream executor produce a better artifact. \modelname{} instead trains a reasoning planner by combining reason-only quality scoring with measured utility for the frozen executor that consumes the trace.

\section{Method}
\label{sec:method}

\modelname{} trains reasoning traces as executor-consumable intermediate artifacts in a fixed planner-executor protocol. In this section, we first describe \datasetname{}, the grouped reason-only supervision data used to learn trace-quality judgments. We then introduce \modelname{} training framework, which uses scores problem-reasoning pairs to train the offline Reasoning Reward Model (RM). Finally, we describe executor-grounded planner optimization, where verifier feedback is combined with an uplift-weighted Reasoning RM score to train the planner with GRPO. %At evaluation time, both the Reasoning RM and uplift estimator are removed, so improvements reflect whether the trained planner better guides the same frozen executor.

\begin{figure*}[t]
\centering
\includegraphics[width=\textwidth]{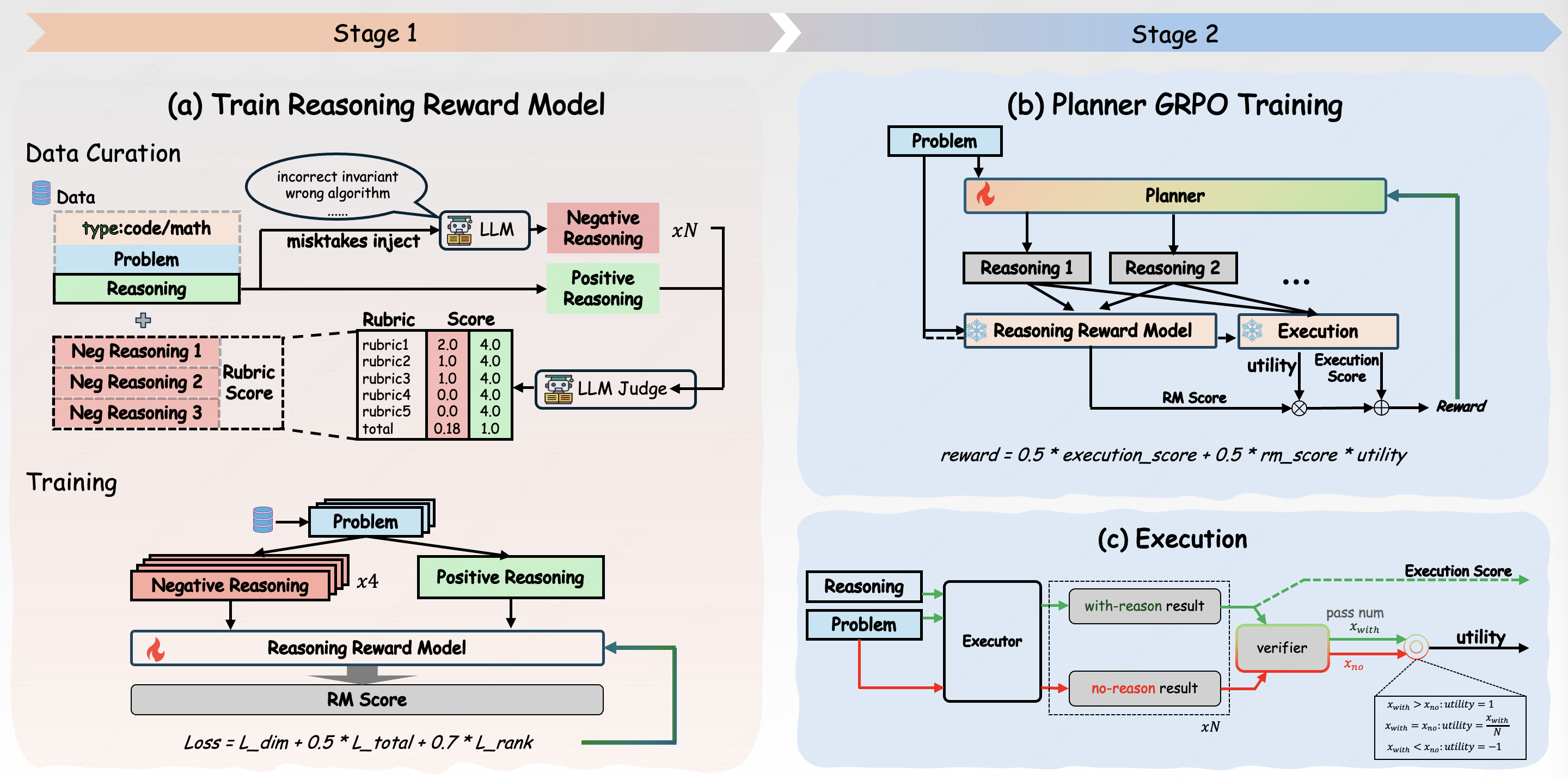}
\caption{The overall framework of \datasetname{} and \modelname{}. \textbf{(a)} Data curation pipeline of \datasetname{}. Then we use \datasetname{} to finetune the reward model specialized for reasoning supervising by the designed loss. \textbf{(b)} GRPO training process of the planner using previous trained reasoning reward model. \textbf{(c)} Details of execution calculation process. The Reasoning RM score is weighted by measured executor uplift before being combined with verifier feedback for planner optimization.}
\label{fig:method-pipeline}
\end{figure*}

% \subsection{Preliminary}

% Planner optimization follows GRPO-style reinforcement learning~\citep{shao2024deepseekmath,deepseek2025r1}. For each problem $P$, the old planner $\pi_{\theta_{\mathrm{old}}}$ samples a group of $G$ reasoning trajectories $\{R_i\}_{i=1}^{G}$. Each trajectory $R_i=(y_{i,1},\ldots,y_{i,T_i})$ receives a scalar reward $r_i=\mathcal{R}(P,R_i)$. GRPO avoids a learned value critic by estimating the baseline from rewards within the same problem group. The group-relative advantage is
% \begin{align}
% \hat{A}_i
% =
% \frac{r_i-\mathrm{mean}_{j\le G}(r_j)}
% {\mathrm{std}_{j\le G}(r_j)+\delta}.
% \label{eq:grpo-adv}
% \end{align}
% The same trajectory-level advantage is assigned to all tokens in $R_i$.

% The planner is updated with the clipped surrogate
% \begin{align}
% \mathcal{J}_{\mathrm{GRPO}}(\theta)
% =
% \mathbb{E}
% \left[
% \frac{1}{G}\sum_{i=1}^{G}
% \frac{1}{T_i}\sum_{t=1}^{T_i}
% \min\left(
% \rho_{i,t}(\theta)\hat{A}_i,
% \clip(\rho_{i,t}(\theta),1-\epsilon,1+\epsilon)\hat{A}_i
% \right)
% -\beta D_{\mathrm{KL}}(\pi_{\theta}\,\|\,\pi_{\mathrm{ref}})
% \right],
% \label{eq:grpo-objective}
% \end{align}
% where
% \begin{align}
% \rho_{i,t}(\theta)
% =
% \frac{
% \pi_{\theta}(y_{i,t}\mid P,y_{i,<t})
% }{
% \pi_{\theta_{\mathrm{old}}}(y_{i,t}\mid P,y_{i,<t})
% }.
% \end{align}
% Here $\epsilon$ is the clipping range, $\delta$ stabilizes group normalization, and the KL term is an optional reference-policy regularizer. In \modelname{}, the reward $\mathcal{R}(P,R)$ is the executor-grounded reasoning reward defined in Section~\ref{sec:tracelift-training}.

\subsection{\datasetname{}}

\datasetname{} provides reason-only supervision for evaluating reasoning trajectories without using final-artifact correctness as the label. Existing outcome or process labels usually do not isolate whether an intermediate trajectory is reliable and useful to the executor that consumes it. We therefore introduce \datasetname{}, a rubric-annotated reason-only dataset. Each example is a reasoning group anchored to one problem, containing a high-quality reference trajectory and multiple plausible flawed trajectories. The dataset is constructed in two stages. We first synthesize targeted trajectory perturbations, and then use an agentic rubric annotation pipeline to assign multi-dimensional reasoning-quality scores.

\subsubsection{Trajectory Perturbation}

We construct \datasetname{} from 3,000 seed problems sampled from OpenCodeReasoning and 3,000 seed problems sampled from the GSM8K training set. Each seed provides a problem, a reference reasoning trajectory, and a reference result. The reference trajectory is normalized into a reason-only trajectory that preserves the intermediate reasoning. 
% while removing final-answer or complete-code leakage. 
% For code, the trajectory should describe the algorithmic route, data structures, invariants, edge cases, and complexity without providing a pasteable implementation. For math, it should preserve the solution strategy and key intermediate derivations without directly revealing the final answer.

As shown in Figure~\ref{fig:method-pipeline}, we generate multiple plausible flawed trajectories through targeted error injection from each reference trajectory. The perturbation is local enough to keep the trajectory task-relevant, but strong enough to reduce reasoning reliability or solution support. Code perturbations include wrong algorithm choices, missing edge cases, off-by-one reasoning, incorrect invariants, infeasible complexity, vague pseudo-solutions, and irrelevant explanations. Math perturbations include arithmetic slips, wrong operations, dropped cases, unit mismatches, unsupported jumps, premature answers, and content-free reasoning. Across the 6,000 reasoning groups, these flawed trajectories cover seven perturbation types per domain. This design turns each problem into a controlled comparison over reasoning quality rather than a binary outcome-label example.

\subsubsection{Agentic Rubric Annotation}

We annotate each trajectory with an Large-Language-Model-based rubric judge that evaluates the reasoning process rather than the final artifact. The judge receives the problem, the candidate's trajectory, and task-specific rubric instructions. It returns dimension scores, a total score, and a short rationale for the rating. For code, the rubric scores task understanding, plan quality, step coherence, action support, and non-leakage. For math, it scores problem understanding, solution strategy, step coherence, calculation correctness, and answer support. Each dimension is mapped to a five-class label, and the total score is normalized to $[0,1]$.

The annotation pipeline is designed to prevent two shortcuts in Reasoning RM training. A trajectory should not receive high reward merely because a later artifact happens to pass, and it should not receive high reward merely because it is fluent or verbose. By placing reference and flawed trajectories under the same problem context, \datasetname{} provides both pointwise rubric supervision and within-group ranking supervision for learning trajectory quality.

\subsection{\modelname{} Training Framework}
\label{sec:tracelift-training}

As shown in Figure~\ref{fig:method-pipeline}, \modelname{} has two training stages. The first stage trains a Reasoning RM to score problem-reasoning pairs from rubric-labeled reasoning groups. The second stage uses this score together with verifier feedback and measured executor uplift to train the planner with GRPO. The Reasoning RM and measured uplift are training-time reward components, while the planner remains the only trainable generator in the planner-executor chain.

\subsubsection{Offline Reasoning RM Training}

We first train a Reasoning RM offline on \datasetname{} before planner optimization. Given a problem $P$ and a reasoning trajectory $R$, the Reasoning RM returns a scalar quality score.
\begin{equation}
\mathrm{RM}_{\phi}(P,R)\rightarrow \srm,\qquad \srm\in[0,1].
\label{eq:rm-score}
\end{equation}
The RM input contains the problem context and the trajectory. Here $P$ denotes the static task context available before artifact generation. It may include the problem statement and any public constraints or tests supplied in the prompt. It excludes executor outputs and verifier outcomes as well as final answers or code. This restriction makes the offline training target rubric-labeled reasoning quality rather than outcome correctness.

The Reasoning RM is applied pointwise but trained with group supervision. For each update, we sample one reference trajectory and four flawed trajectories from the same reasoning group. Each trajectory is encoded independently as a problem-reasoning input, while the group structure is used only by the ranking loss. The model uses five rubric heads to predict discrete dimension labels and an aggregate head to predict the normalized total score. For dimension \(k\), the annotated rubric score \(\tilde{y}^{\star}_{k}\) is rounded and clipped into a five-class label:
\begin{equation}
y^{\star}_{k}
=
\clip\!\left(
\left\lfloor \tilde{y}^{\star}_{k}+\frac{1}{2}\right\rfloor,
0,4
\right),
\qquad
\mathcal{L}_{\mathrm{dim}}
=
-\frac{1}{5}\sum_{k=1}^{5}\log p_{k,y^{\star}_{k}} .
\label{eq:rm-dim-loss}
\end{equation}
Let \(p_{k,c}\) be the predicted probability that rubric dimension \(k\) receives label \(c\in\{0,\ldots,4\}\). We compute the rubric-head-derived score as
\begin{equation}
\srm^{\mathrm{dim}}
=
\sum_{k=1}^{5} w_k
\sum_{c=0}^{4} p_{k,c}\frac{c}{4},
\qquad
\sum_{k=1}^{5}w_k=1.
\label{eq:rm-dim-score}
\end{equation}
The aggregate head outputs \(\srm^{\mathrm{total}}=\sigma(z_{\mathrm{total}})\), and is trained with a Huber loss against the clipped normalized total score:
\begin{equation}
\mathcal{L}_{\mathrm{total}}
=
\mathrm{Huber}_{\delta=1}
\!\left(
\srm^{\mathrm{total}}
-
\clip(\tilde{y}_{\mathrm{total}}^{\star},0,1)
\right).
\label{eq:rm-total-loss}
\end{equation}
The online RM score combines \(\srm^{\mathrm{dim}}\) with the aggregate score:
\begin{equation}
\srm
=
0.5\,\srm^{\mathrm{dim}}
+ 0.5\,\srm^{\mathrm{total}}.
\label{eq:rm-score-aggregation}
\end{equation}
We optimize
\begin{equation}
\mathcal{L}_{\mathrm{RM}}
=
\mathcal{L}_{\mathrm{dim}}
+ \alpha \mathcal{L}_{\mathrm{total}}
+ \beta \mathcal{L}_{\mathrm{rank}},
\qquad
\mathcal{L}_{\mathrm{rank}}
=
-\log \sigma(\srm^{+}-\srm^{-}).
\label{eq:rm-loss}
\end{equation}
The ranking pairs are drawn from the same reasoning group. In the current implementation, we set \(\alpha=0.5\) and \(\beta=0.7\).

\subsubsection{Executor-Grounded Planner Optimization}

We then train the planner to produce trajectories that are both rubric-good and useful to the frozen executor. For each problem, the planner samples reasoning trajectories using the GRPO objective from the preliminary. For each sampled trajectory $R_i$, the frozen executor produces a final artifact $\Aout_i$, and the verifier scores that artifact.
\begin{align}
\Aout_i = E(P,R_i), \nonumber\\
\Rexec(P,R_i) = V(\Aout_i),
\label{eq:executor-reward}
\end{align}
% where \(R_i \sim \pi_{\theta}(\cdot\mid P)\).
where \(R_i \sim \pi_\theta(\cdot\mid P)\), and \(\Rexec(P,R_i)\) is the executor reward assigned to the sampled trajectory, and $V$ denotes the task verifier. To ground the RM score in executor utility, \modelname{} also compares the same frozen executor with and without the planner trajectory. For uplift measurement, $V\in\{0,1\}$ denotes binary verifier success under matched decoding settings. We estimate the success rate with trajectory $R$ as $\hat{p}(P,R)$ and the no-reasoning baseline success rate as $\hat{p}_{0}(P)$.
\begin{align}
\hat{p}(P,R)
&=
\frac{1}{K}\sum_{k=1}^{K}V(E_k(P,R)), \nonumber\\
\hat{p}_{0}(P)
&=
\frac{1}{K}\sum_{k=1}^{K}V(E_k(P,\varnothing)), \nonumber\\
\urm(P,R)
&=
\clip\big(\hat{p}(P,R)-\hat{p}_{0}(P),\ -1,\ 1\big).
\label{eq:uplift}
\end{align}
where the trajectory $R$ is fixed over repeated executor samples during estimation. No-reasoning baseline uses the same executor prompt with the reasoning field omitted. Positive uplift means that $R$ improves verifier success for this executor, while negative uplift means that it hurts the executor.

The scalar reward therefore combines verifier feedback, Reasoning RM quality and executor uplift:
\begin{equation}
\mathcal{R}(P,R)
=
0.5\,\Rexec(P,R)
+ 0.5\,\mathrm{RM}_{\phi}(P,R)\,\urm(P,R).
\label{eq:reward}
\end{equation}
The verifier term anchors planner training to task success. The uplift-weighted RM term increases the reward when a high-rubric trajectory also improves the frozen executor. This combined scalar reward is assigned at the completion level and used in the GRPO objective. At test time, the Reasoning RM and measured uplift are removed, and Section~\ref{sec:experiments} describes the resulting evaluation protocol.

\section{Experiments}
\label{sec:experiments}

% We evaluate whether executor-grounded reasoning reward improves a fixed planner-executor system over execution-only RL. The experiments are designed to compare training objectives under the same evaluation chain, not to claim direct state-of-the-art performance for a single model.

\subsection{Experimental setup}

\paragraph{Models and training objectives.}
We train \modelname{} on Qwen2.5-7B, Llama3.1-8B, and Qwen3-4B model families~\citep{yang2024qwen2,grattafiori2024llama3,yang2025qwen3}.
% The baselines are: Base, the untrained planner in the same two-stage pipeline;  
For the main policy Group Relative Policy Optimization (GRPO) experiments, all Qwen2.5-7B, Llama3.1-8B, and Qwen3-4B policies are trained for 600 steps in bf16 with learning rate $5\times10^{-6}$, and we use a temperature of 0.5. 
% \modelname{} combines the policy's end-to-end reward with an uplift-weighted Reason Reward Model(RM) score using equal weights. For completed code runs, we additionally report a rollout-count sweep that varies the number of executor comparisons used to estimate uplift.

% \paragraph{Benchmarks.}
% For code, we evaluate HumanEval~\citep{chen2021evaluating}, HumanEvalPlus from EvalPlus~\citep{liu2024evalplus}, MBPP-full~\citep{austin2021program}, and LiveCodeBench public tests~\citep{jain2024livecodebench}. The MBPP-full setting uses the 500-example full/test split and is not the EvalPlus MBPP/MBPP+ leaderboard subset. The LiveCodeBench results use local public tests. For math, we evaluate GSM8K, GSM-Hard, SVAMP, and MATH500. We report aggregate scores as micro-averages over the listed benchmark instances.
\paragraph{Evaluation metrics.}
For code, we evaluate \modelname{} on HumanEval~\citep{chen2021evaluating}, HumanEval+ ~\citep{liu2024evalplus}, MBPP-full~\citep{austin2021program} and LiveCodeBench~\citep{jain2024livecodebench}. 
% The MBPP-full setting uses the 500-example full/test split. The LiveCodeBench results use local public tests. 
% The code micro-average is computed over the listed benchmark instances. 
For math, we evaluate GSM8K~\citep{gsm8k}, GSM-Hard~\citep{gsmhard}, SVAMP~\citep{svamp} and MATH500~\citep{math500}. 
% The math micro-average is computed over the listed benchmark instances. 

\paragraph{Evaluation protocols.}
% The main protocol is two-stage evaluation. 
For each problem, the trained policy first generates a reasoning trace with greedy decoding. The frozen executor then consumes the problem and reasoning and generates the final artifact. Code artifacts are judged by executable tests, and math artifacts are judged by answer matching. The Reason RM and uplift estimator are not called during evaluation. For math only, we also report a supplementary direct self-solving protocol in which the policy itself generates both reasoning and final answer; this protocol is not the main evidence for the planner-executor claim.

\subsection{Main Results}

\paragraph{Training details.}
% For the main policy Group Relative Policy Optimization(GRPO) experiments, all Qwen2.5-7B, Llama3.1-8B, and Qwen3-4B policies are trained for 600 steps in bf16 with learning rate $5\times10^{-6}$. Both code and math runs use the same LoRA configuration: rank 16, alpha 32, dropout 0.05, targeting \texttt{q\_proj}, \texttt{k\_proj}, \texttt{v\_proj}, \texttt{o\_proj}, \texttt{gate\_proj}, \texttt{up\_proj}, and \texttt{down\_proj}. Full-parameter Qwen2.5-7B results are reported separately in Table~\ref{tab:qwen25-full}.

Compared with \modelname{}, Exec-only only trains the policy with end-to-end verifier reward instead of being trained with Eq.~\ref{eq:reward}. 

\paragraph{Results on Code Benchmarks.}
Table~\ref{tab:code-math-main} reports code results under the fixed two-stage protocol. Across all three model families, \method{} improves over Exec-only on every benchmark. On the micro-average, \method{} improves over Exec-only by $2.61$ percentage points for Qwen2.5-7B, $1.96$ points for Llama3.1-8B, and $2.44$ points for Qwen3-4B. These gains are especially meaningful because the executor is frozen during evaluation: the only component changed is the planner that produces the reasoning trace.

The code results show two complementary effects. First, execution-only RL already improves over the untrained two-stage baseline, indicating that verifier feedback is a useful signal for planner training. Second, \method{} consistently improves over Exec-only, showing that final execution reward alone is not sufficient. The improvement is largest on the harder LiveCodeBench public-test setting for Qwen2.5-7B and Llama3.1-8B, where small changes in the plan often determine whether the executor handles constraints, edge cases, and implementation details correctly. This supports the central claim that reasoning traces should be rewarded not only for downstream correctness, but also for whether they provide executor-consumable guidance.

\begin{table*}[t]
\centering
\scriptsize
\setlength{\tabcolsep}{4.2pt}
\renewcommand{\arraystretch}{1.08}
\caption{Two-stage results on code and math benchmarks. We use pass@1 percentages for code and accuracy for math entries to assess. During evaluation, the executor is initialized from the same model family as the planner. HE and HE+ denote HumanEval and HumanEval+, LCB denotes LiveCodeBench. All results reported in the table are averaged over three runs with different seeds. Micro avg. denotes the average accuracy of four benchmarks respectively.}
\vspace{2mm}
\label{tab:code-math-main}
\begin{tabular}{llccccc|ccccc}
\toprule
\multirow{2}{*}{Model}
& \multirow{2}{*}{Method}
& \multicolumn{5}{c|}{Code}
& \multicolumn{5}{c}{Math} \\
\cmidrule(lr){3-7}
\cmidrule(lr){8-12}
& 
& HE
& HE+
& MBPP
& LCB
& Micro avg.
& GSM8K
& GSM-Hard
& SVAMP
& MATH500
& Micro avg. \\
\midrule

\multirow{3}{*}{Qwen2.5-7B}
& Base
& 68.90 & 60.37 & 58.00 & 29.50 & 50.49
& 61.79 & 30.02 & 58.67 & 42.40 & 46.51 \\
& Exec-only
& 70.12 & 61.59 & 59.00 & 32.75 & 52.28
& 87.04 & 47.84 & 68.67 & 48.00 & 64.72 \\
& \method{}
& \best{72.56} & \best{64.02} & \best{61.20} & \best{36.00} & \best{54.89}
& \best{89.16} & \best{51.78} & \best{89.67} & \best{50.40} & \best{69.23} \\

\midrule

\multirow{3}{*}{Llama3.1-8B}
& Base
& 33.54 & 29.88 & 34.40 & 10.25 & 25.81
& 16.60 & 4.93 & 34.33 & 10.20 & 12.74 \\
& Exec-only
& 37.80 & 31.10 & 45.20 & 15.00 & 32.49
& 32.15 & 9.45 & 48.67 & 13.40 & 22.13 \\
& \method{}
& \best{39.02} & \best{32.93} & \best{46.40} & \best{18.25} & \best{34.45}
& \best{35.71} & \best{9.70} & \best{49.00} & \best{14.60} & \best{23.82} \\

\midrule

\multirow{3}{*}{Qwen3-4B}
& Base
& 79.27 & 75.61 & 63.00 & 56.50 & 64.74
& 88.48 & 52.16 & 86.00 & 67.00 & 71.20 \\
& Exec-only
& 81.09 & 76.83 & 63.60 & 58.00 & 65.88
& 89.01 & 52.01 & 88.67 & 66.40 & 71.50 \\
& \method{}
& \best{84.15} & \best{78.66} & \best{65.80} & \best{60.75} & \best{68.32}
& \best{89.23} & \best{53.22} & \best{89.00} & \best{68.20} & \best{72.34} \\

\bottomrule
\end{tabular}
\end{table*}

\paragraph{Results on Math Benchmarks.}
Table~\ref{tab:code-math-main} reports math results under the same two-stage planner-executor protocol. \method{} again improves over Exec-only across all model families and all listed benchmarks. On the micro-average, \method{} improves over Exec-only by $4.51$ percentage points for Qwen2.5-7B, $1.69$ points for Llama3.1-8B, and $0.84$ points for Qwen3-4B.

The math results are consistent with the code results but reveal a stronger dependence on model strength. For Qwen2.5-7B, \method{} gives a large micro-average gain over Exec-only, with particularly strong improvement on SVAMP and GSM-Hard. For Qwen3-4B, the base two-stage system is already strong, so the remaining headroom is smaller; nevertheless, \method{} still improves over Exec-only on every listed benchmark. This pattern suggests that executor-grounded reasoning rewards are most valuable when the planner must supply missing structure to a capable but imperfect executor, while still providing gains for stronger base systems.

\subsection{LoRA and Full-Parameter GRPO}

Table~\ref{tab:qwen25-full} compares LoRA and full-parameter GRPO for Qwen2.5-7B. On code, full-parameter training improves both objectives, but \method{} remains stronger than Exec-only under the same parameterization. In the full-parameter setting, \method{} reaches a code micro-average of $60.26\%$, compared with $57.90\%$ for Exec-only. On math, full-parameter training is not uniformly better than LoRA, but the objective-level comparison remains favorable: \method{} improves over Exec-only under both LoRA and full-parameter training.

These results indicate that the gains from \method{} are not merely a byproduct of parameter budget. On code, increasing trainable capacity raises the ceiling, and the executor-grounded reward continues to provide an additional gain over execution-only training. On math, LoRA \method{} is the strongest overall configuration, while full-parameter \method{} still outperforms full-parameter Exec-only. This suggests that the training objective and the adaptation regime play distinct roles: more trainable parameters can help, but they do not replace the need for a reward that distinguishes useful reasoning from reasoning that merely correlates with final correctness.

\subsection{Supplementary Math Direct Self-Solving}

The main evaluation fixes the executor and measures whether the trained planner produces better executor-consumable reasoning. A possible concern is that such training might overspecialize the policy to the frozen executor and harm its own direct problem-solving ability. To test this, we evaluate math policies in a direct self-solving protocol where the policy generates both the reasoning and final answer. Table~\ref{tab:math-direct} reports the results.

% \begin{table*}[t]
% \centering
% \small
% \setlength{\tabcolsep}{6pt}
% \renewcommand{\arraystretch}{1.08}
% \caption{Supplementary math direct self-solving results. All entries are answer accuracy percentages. In this protocol, the policy generates both reasoning and final answer while no frozen executor is used at evaluation time.}
% \label{tab:math-direct}
% \vspace{2mm}
% \begin{tabular}{llrrr}
% \toprule
% Model & Benchmark & Base & Exec-only & \method{} \\
% \midrule
% \multirow{5}{*}{Qwen2.5-7B}
% & GSM8K & 61.79 & 87.26 & \best{89.31} \\
% & GSM-Hard & 30.02 & 47.92 & \best{52.16} \\
% & SVAMP & 58.67 & 67.67 & \best{90.33} \\
% & MATH500 & 42.40 & 54.00 & \best{55.20} \\
% & Micro avg. & 46.51 & 65.62 & \best{70.19} \\
% \midrule
% \multirow{5}{*}{Llama3.1-8B}
% & GSM8K & 16.60 & 31.92 & \best{32.90} \\
% & GSM-Hard & 4.93 & 9.78 & \best{10.00} \\
% & SVAMP & 34.33 & 46.33 & \best{47.67} \\
% & MATH500 & 10.20 & 14.80 & \best{15.40} \\
% & Micro avg. & 12.74 & 22.19 & \best{22.86} \\
% \midrule
% \multirow{5}{*}{Qwen3-4B}
% & GSM8K & 88.48 & 88.93 & \best{89.16} \\
% & GSM-Hard & 52.16 & 52.39 & \best{52.69} \\
% & SVAMP & 86.00 & \best{88.67} & \best{88.67} \\
% & MATH500 & 67.00 & 68.60 & \best{69.40} \\
% & Micro avg. & 71.20 & 71.93 & \best{72.25} \\
% \bottomrule
% \end{tabular}
% \end{table*}
\begin{table*}[t]
\centering
\scriptsize
\setlength{\tabcolsep}{5pt}
\renewcommand{\arraystretch}{1.08}
\caption{Supplementary math direct self-solving results. All entries are answer accuracy percentages. In this protocol, the policy generates both reasoning and final answer while no frozen executor is used at evaluation time. All results reported in the table are averaged over three runs with different seeds.}
\label{tab:math-direct}
\resizebox{\textwidth}{!}{%
\begin{tabular}{llrrrrr}
\toprule
Model & Method & GSM8K & GSM-Hard & SVAMP & MATH500 & Micro avg. \\
\midrule

\multirow{3}{*}{Qwen2.5-7B}
& Base
& 61.79 & 30.02 & 58.67 & 42.40 & 46.51 \\
& Exec-only
& 87.26 & 47.92 & 67.67 & 54.00 & 65.62 \\
& \method{}
& \best{89.31} & \best{52.16} & \best{90.33} & \best{55.20} & \best{70.19} \\

\midrule

\multirow{3}{*}{Llama3.1-8B}
& Base
& 16.60 & 4.93 & 34.33 & 10.20 & 12.74 \\
& Exec-only
& 31.92 & 9.78 & 46.33 & 14.80 & 22.19 \\
& \method{}
& \best{32.90} & \best{10.00} & \best{47.67} & \best{15.40} & \best{22.86} \\

\midrule

\multirow{3}{*}{Qwen3-4B}
& Base
& 88.48 & 52.16 & 86.00 & 67.00 & 71.20 \\
& Exec-only
& 88.93 & 52.39 & \best{88.67} & 68.60 & 71.93 \\
& \method{}
& \best{89.16} & \best{52.69} & \best{88.67} & \best{69.40} & \best{72.25} \\

\bottomrule
\end{tabular}%
}
\end{table*}

\begin{table}[t]
\centering
\small
\setlength{\tabcolsep}{6pt}
\renewcommand{\arraystretch}{1.08}
\caption{Ablation on the number of executor comparisons $K$ used to estimate uplift during Qwen2.5-7B code training. All entries are pass@1 percentages.}
\label{tab:rollout-sweep}
\begin{tabular}{lrrr}
\toprule
Benchmark & $K=1$ & $K=3$ & $K=5$ \\
\midrule
HumanEval & 68.29 & \best{72.56} & 67.07 \\
HumanEval+ & 63.41 & \best{64.02} & 60.37 \\
MBPP-full & 60.80 & \best{61.20} & 60.40 \\
LiveCodeBench & 35.00 & \best{36.00} & 35.50 \\
Micro avg. & 53.75 & \best{54.89} & 53.18 \\
\bottomrule
\end{tabular}
\end{table}

The direct self-solving results rule out a simple negative-transfer explanation. \method{} improves or matches Exec-only on all listed math direct-solving settings, including a $4.57$ point micro-average gain for Qwen2.5-7B. Thus, the executor-grounded reward does not appear to train traces that are only useful to an external executor while degrading the policy's own ability to complete the solution. Instead, the learned reasoning remains useful when the same policy must also produce the final answer. We nevertheless treat this as supplementary evidence, since the main claim of this paper concerns the fixed planner-executor setting.

\subsection{Ablation Studies}

We run ablations on Qwen2.5-7B code tasks, where executable tests provide a high-precision verifier and the planner-executor separation is most direct.

\paragraph{Number of executor comparisons.}
The uplift estimate in Eq.~\ref{eq:uplift} depends on the number of executor comparisons used during training. Table~\ref{tab:rollout-sweep} varies this number. The main setting uses $K=3$.

The sweep shows that the uplift estimator is useful but not monotonically improved by more comparisons. With $K=1$, the reward can be too sensitive to a single executor outcome, so the planner may receive high credit for traces that happen to help one executor comparison but do not robustly guide the executor. Increasing to $K=3$ stabilizes this marginal-utility signal and gives the best aggregate performance. However, $K=5$ is weaker in this setting. A likely reason is that a more averaged uplift signal can become conservative: small positive and negative effects cancel, weakening the RM-weighted term that distinguishes actionable traces inside a GRPO group. Since evaluation is pass@1 with deterministic decoding, optimizing an overly smoothed training-time utility estimate is not necessarily aligned with the final evaluation target. We therefore use $K=3$ as a practical balance between reward stability, cost, and pass@1 alignment.

\begin{table*}[t]
\centering
\scriptsize
\setlength{\tabcolsep}{4pt}
\renewcommand{\arraystretch}{1.08}
\caption{Qwen2.5-7B LoRA and full-parameter GRPO results under the two-stage protocol. All results reported in the table are averaged over three runs with different seeds.}
\vspace{2mm}
\label{tab:qwen25-full}
\begin{tabular}{llrrrrr}
\toprule
Domain & Benchmark & Base & Exec-only LoRA & Exec-only Full & \method{} LoRA & \method{} Full \\
\midrule
\multirow{5}{*}{Code}
& HumanEval & 68.90 & 70.12 & 73.78 & 72.56 & \best{76.83} \\
& HumanEval+ & 60.37 & 61.59 & 68.29 & 64.02 & \best{70.12} \\
& MBPP-full & 58.00 & 59.00 & 61.00 & 61.20 & \best{63.60} \\
& LiveCodeBench & 29.50 & 32.75 & 43.25 & 36.00 & \best{45.25} \\
& Micro avg. & 50.49 & 52.28 & 57.90 & 54.89 & \best{60.26} \\
\midrule
\multirow{5}{*}{Math}
& GSM8K & 61.79 & 87.04 & 86.66 & \best{89.16} & 87.79 \\
& GSM-Hard & 30.02 & 47.84 & 48.29 & \best{51.78} & 49.28 \\
& SVAMP & 58.67 & 68.67 & 83.67 & \best{89.67} & 89.33 \\
& MATH500 & 42.40 & 48.00 & 48.80 & 50.40 & \best{52.00} \\
& Micro avg. & 46.51 & 64.72 & 66.17 & \best{69.23} & 67.95 \\
\bottomrule
\vspace{-2mm}
\end{tabular}
\end{table*}

\paragraph{Reward component ablations.}
Table~\ref{tab:reward-ablation} isolates the two key components of the \method{} reward. The No-uplift variant removes executor grounding from the RM term:
\[
\mathcal{R}(P,R)
=
0.5\,\Rexec(P,R)
+
0.5\,\mathrm{RM}_{\phi}(P,R).
\]
The RM-uplift only variant removes the end-to-end verifier anchor:
\[
\mathcal{R}(P,R)
=
\mathrm{RM}_{\phi}(P,R)\,\urm(P,R).
\]

\begin{table*}[t]
\centering
\small
\setlength{\tabcolsep}{4pt}
\renewcommand{\arraystretch}{1.08}
\caption{Ablation results of Qwen2.5-7B code reward. All entries are pass@1 percentages. Full \method{} combines verifier reward with uplift-weighted reasoning reward model.}
\vspace{2mm}
\label{tab:reward-ablation}
\begin{tabular}{lrrrrr}
\toprule
Benchmark & Base & Exec-only & \method{} & No-uplift & RM-uplift only \\
\midrule
HumanEval & 68.90 & 70.12 & \best{72.56} & 65.85 & 71.34 \\
HumanEval+ & 60.37 & 61.59 & \best{64.02} & 60.37 & 63.41 \\
MBPP-full & 58.00 & 59.00 & \best{61.20} & 58.40 & 58.60 \\
LiveCodeBench & 29.50 & 32.75 & \best{36.00} & 33.50 & 33.25 \\
Micro avg. & 50.49 & 52.28 & \best{54.89} & 51.55 & 52.69 \\
\bottomrule
\vspace{-2mm}
\end{tabular}
\end{table*}

\begin{figure*}[t]
\centering
\includegraphics[width=0.98\textwidth]{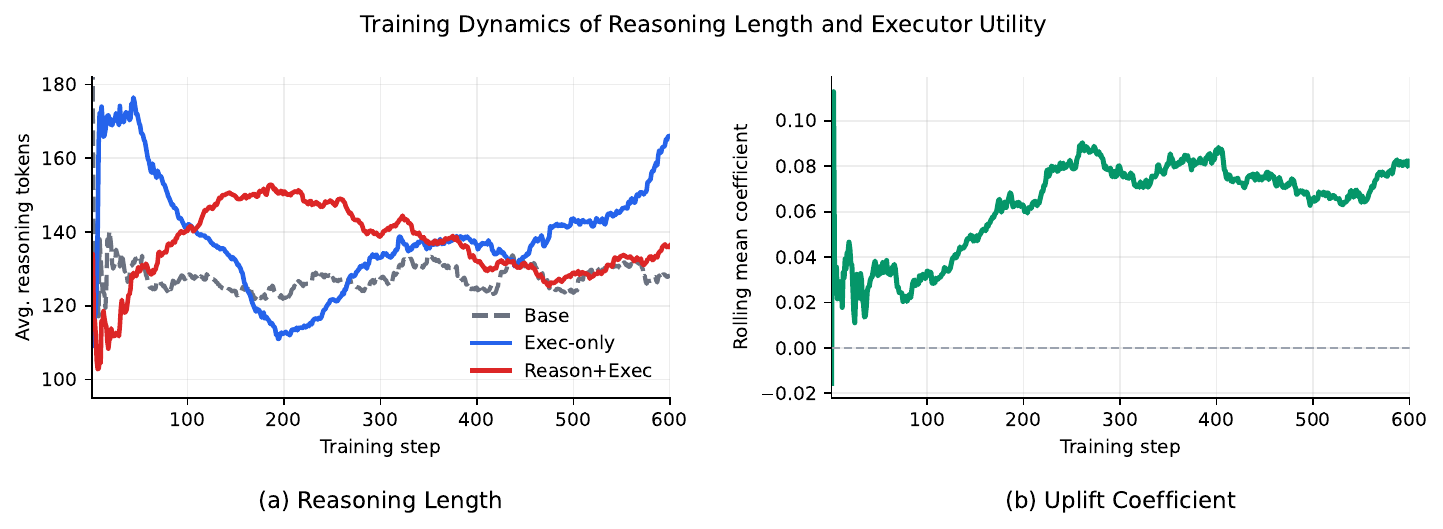}
\caption{Training dynamics of reasoning length and executor utility. \textbf{Left}: rolling mean of reasoning length for differend methods. The Base is evaluated on the same problem sequence and does not reflect parameter updates. \textbf{Right}: rolling mean of the uplift coefficient used in the \method{} reward.}
\label{fig:length-uplift}
\vspace{-2mm}
\end{figure*}

These ablations show that both reward axes are necessary. Removing uplift turns the RM term into an intrinsic rubric score detached from the executor. This variant is weaker than the full method on every benchmark and even falls below Exec-only on the micro-average. The failure mode is intuitive: a trace can be coherent and rubric-plausible while still being too abstract, misleading, or poorly matched to the executor prompt. Without measured uplift, high-scoring traces that do not improve executor success are not penalized.

The RM-uplift only variant has the opposite failure. It preserves executor grounding but removes the verifier anchor, so it can reward traces that improve over the no-reasoning baseline while still producing incorrect artifacts. Such relative improvement is useful, but not task success. This is visible on MBPP-full and LiveCodeBench, where RM-uplift only improves over Base but remains far below full \method{}. The full reward resolves both failures: the verifier term anchors optimization to correct artifacts, while the RM-uplift term filters reasoning credit through rubric quality and executor utility.

\subsection{Analysis: Reasoning Length and Executor Utility}

A natural alternative explanation is that \method{} improves performance by simply making the planner produce longer reasoning traces. We test this by measuring reasoning length during training. Table~\ref{tab:length-analysis} summarizes the last 150 training steps, corresponding to 300 on-policy reasoning samples. The Base row is computed by running the Qwen2.5-7B base planner on the same sequence of training problems, so it reflects problem-sequence variation rather than parameter updates.

\begin{table}[t]
\centering
\small
\setlength{\tabcolsep}{7pt}
\renewcommand{\arraystretch}{1.08}
\caption{Reasoning length statistics over the final training window. Token counts are computed with the Qwen2.5 tokenizer. The bin columns report percentages over 300 sampled reasoning traces.}
\label{tab:length-analysis}
\begin{tabular}{lrrrr}
\toprule
Model & Avg. tokens & Median & $0$--$128$ tokens & $\geq256$ tokens \\
\midrule
Base & 128.2 & 104.0 & 59.00 & 7.00 \\
Exec-only & 166.0 & 147.5 & 38.67 & 12.67 \\
\method{} & 136.3 & 125.5 & 53.67 & 5.00 \\
\bottomrule
\end{tabular}
\vspace{-2mm}
\end{table}

The length statistics and training curves reject the verbosity explanation. Exec-only produces longer reasoning than Base, averaging $166.0$ tokens with a larger long-tail fraction. In contrast, \method{} improves task performance while keeping the average length at $136.3$ tokens, nearly $30$ tokens shorter than Exec-only. Thus, the gain does not come from longer traces or format preference.

Figure~\ref{fig:length-uplift} further shows that \method{} increases executor utility without increasing reasoning length. The uplift coefficient gradually becomes more positive during training, while reasoning length remains moderate. This indicates that the planner learns to include information that changes executor behavior, rather than merely expanding the explanation. In other words, the reward shifts the planner toward compact, executor-useful traces.

% \subsection{Summary of Findings}

% Across code and math, \method{} improves over execution-only RL under the same fixed planner-executor evaluation chain. The gains hold across Qwen2.5-7B, Llama3.1-8B, and Qwen3-4B, and the supplementary math direct-solving results show no evidence that executor-grounded reasoning training harms the policy's own solving ability. The ablations show that Reason RM scores must be grounded by measured executor uplift, and that uplift-weighted reasoning quality still needs the verifier reward as an end-to-end anchor. Finally, the length analysis shows that \method{} does not obtain gains by simply producing longer traces; instead, it learns more compact reasoning that better guides the frozen executor.

\section{Conclusion}
\label{sec:conclusion}

In this paper, we introduced \modelname{}, a planner-executor framework that trains reasoning trajectories as executor-consumable intermediate artifacts. \modelname{} trains a reasoning planner with verifier feedback on frozen-executor outputs and a Reason RM score weighted by measured executor uplift. We also introduced \datasetname{}, a rubric-annotated reason-only dataset that pairs reference trajectories with targeted flawed trajectories for learning intermediate reasoning quality. Experiments on code and math benchmarks show that this training recipe improves fixed two-stage planner-executor systems over execution-only reinforcement learning. Together, these results highlight a practical direction for reasoning supervision: reward trajectories not only by how they read, but by whether they are accurate enough to help the executor produce a correct final artifact.

% \section{Additional Details}
% \label{app:details}

% \subsection{Reasoning group examples}

% The current data inspection contains representative code and math groups. In code, one pair asks whether a sequence has pairwise distinct elements: the positive reasoning uses set cardinality, while a wrong-algorithm negative only checks adjacent duplicates and misses non-adjacent repeats. Another code pair concerns maximum team formation and contrasts the correct upper bound $\min(c,m,\lfloor(c+m+x)/3\rfloor)$ with an incorrect constraint model. In math, examples include a wrong-operation negative that adds rather than multiplies, and a dropped-case negative that omits the final medication dose in a multi-step arithmetic problem.

% \subsection{Planned analyses}

% Before a submission version, we plan to add: paired significance tests, RM-score and utility correlations, win/loss flip analysis, reasoning-length controls, dataset scale statistics, leakage audits, and judge-consistency measurements.
\section{Dataset Construction Prompts}
\label{app:dataset-prompts}

This appendix gives the concrete construction templates used to build the reason-only supervision data. We separate three operations: deterministic reason-only cleaning, localized flawed-trace generation, and rubric annotation. The cleaning stage is rule-based rather than LLM-generated; we therefore present it as an implementation specification instead of a prompt. The flawed-trace and rubric stages use the prompt templates shown below.

\subsection{Reason-Only Cleaning Specifications}
\label{app:reason-only-cleaning}

\begin{specbox}{Code reason-only cleaning specification}
\textbf{Purpose.} Convert a raw code reasoning passage into an answer-free planning trace.

\textbf{Preserve.}
\begin{itemize}[leftmargin=1.2em, itemsep=1pt, topsep=2pt]
    \item algorithmic route and decomposition;
    \item data structures, invariants, and state updates;
    \item boundary cases, type constraints, and complexity considerations;
    \item executor-actionable implementation constraints.
\end{itemize}

\textbf{Remove or truncate.}
\begin{itemize}[leftmargin=1.2em, itemsep=1pt, topsep=2pt]
    \item complete implementation leakage such as code fences or pasteable code;
    \item strong code markers such as \texttt{def}, \texttt{class}, \texttt{return}, and \texttt{import} when filtering is enabled;
    \item the exact reference implementation if it appears in the reasoning;
    \item hardcoded test leakage and post-hoc explanation sections.
\end{itemize}

\textbf{Deterministic rules.}
Cut at \texttt{</think>} if present; cut before the exact reference action if it appears; cut before triple-backtick code fences; cut before explanation markers such as \texttt{\#\#\# Explanation}, \texttt{\#\# Explanation}, and \texttt{**Explanation:**}; strip \texttt{<think>} markers and surrounding whitespace.
\end{specbox}

\begin{specbox}{Math reason-only cleaning specification}
\textbf{Purpose.} Convert a raw math solution into an answer-free derivation trace.

\textbf{Preserve.}
\begin{itemize}[leftmargin=1.2em, itemsep=1pt, topsep=2pt]
    \item solution strategy and variable definitions;
    \item key intermediate derivations and equation transformations;
    \item unit conversions and necessary intermediate quantities.
\end{itemize}

\textbf{Remove.}
\begin{itemize}[leftmargin=1.2em, itemsep=1pt, topsep=2pt]
    \item GSM8K-style final-answer markers such as \texttt{\#\#\#\# final\_answer};
    \item direct final-answer leakage;
    \item answer-only shortcuts that do not support the derivation.
\end{itemize}

\textbf{Deterministic rules.}
Remove the GSM8K final-answer marker; replace calculator annotations of the form \texttt{<<expr>>value} with \texttt{expr}; strip trailing blank lines. Synthetic flawed math traces are loaded from the cleaned negative field and then stripped.
\end{specbox}

\subsection{Localized Flawed-Trace Generation}
\label{app:flawed-generation}

The flawed-trace generator receives a cleaned reference reasoning passage and injects one localized perturbation. The goal is to produce a plausible, task-relevant trace whose reasoning quality or solution support is weakened without turning it into random text.

\subsubsection{Code Flawed-Trace Generation}

\begin{table}[h]
\centering
\small
\setlength{\tabcolsep}{5pt}
\renewcommand{\arraystretch}{1.05}
\caption{Code perturbation types used for flawed reasoning generation.}
\label{tab:app-code-perturbations}
\begin{tabularx}{\linewidth}{lY}
\toprule
Perturbation type & Mutation target \\
\midrule
\texttt{wrong\_algorithm\_choice} & Mutate the high-level algorithm or problem decomposition. \\
\texttt{missing\_edge\_case} & Remove or weaken boundary handling and case coverage. \\
\texttt{off\_by\_one} & Corrupt index arithmetic or inclusive/exclusive range reasoning. \\
\texttt{incorrect\_invariant} & Corrupt a maintained state assumption, loop invariant, or proof invariant. \\
\texttt{complexity\_unaware\_plan} & Ignore input-scale constraints or propose an infeasible plan. \\
\texttt{pseudo\_solution\_without\_executable\_detail} & Replace actionable details with vague pseudo-solution text. \\
\texttt{verbose\_irrelevant\_explanation} & Replace useful reasoning with verbose off-target discussion. \\
\bottomrule
\end{tabularx}
\end{table}

\begin{promptlisting}{Code flawed-trace generation: system prompt}
You will be given one passage.

You need to modify only a small part of this passage and inject an error of the specified type.

You should keep all other parts unchanged as much as possible.

Important requirements:
- Do not rewrite the whole passage.
- Only make small local changes.
- Preserve most of the original wording, order, structure, and reasoning style.
- Keep the same format as the input passage.
- Change only the part needed to inject the requested error type.
- Keep all other parts unchanged whenever possible.
- Do not add meta commentary.
- Do not explain what error you introduced.
- Do not mention that the reasoning is wrong.
- Do not point out contradictions, inconsistencies, or failures in the modified reasoning.
- Write the modified passage as if it were a normal reasoning attempt.
- Do not output code blocks.
- Do not add code or final answers.
- Directly output the modified reasoning passage only.
\end{promptlisting}

\begin{promptlisting}{Code flawed-trace generation: user template}
The error type is: {display_name}
This error type means: {definition}

The passage is:
{raw_reason}
\end{promptlisting}

\begin{specbox}{Code flawed-trace parsing and filtering}
The returned text is wrapped as the mutated reasoning. We reject empty generations and generations containing leakage markers such as triple-backtick code fences, \texttt{def}, \texttt{class}, \texttt{<call>}, or \texttt{final code}. During group construction, we additionally drop placeholder negatives, near-duplicates, too-short negatives, and examples without a recognized perturbation type.
\end{specbox}

\subsubsection{Math Flawed-Trace Generation}

\begin{table}[h]
\centering
\small
\setlength{\tabcolsep}{5pt}
\renewcommand{\arraystretch}{1.05}
\caption{Math perturbation types used for flawed reasoning generation.}
\label{tab:app-math-perturbations}
\begin{tabularx}{\linewidth}{lY}
\toprule
Perturbation type & Mutation target \\
\midrule
\texttt{arithmetic\_slip} & Corrupt a small but consequential numeric calculation. \\
\texttt{wrong\_operation} & Use the wrong operation, formula, or equation setup. \\
\texttt{dropped\_case} & Omit a required condition, case, entity, or final step. \\
\texttt{unit\_mismatch} & Confuse units, rates, time spans, or quantities. \\
\texttt{unsupported\_jump} & Jump to a conclusion without the needed intermediate support. \\
\texttt{premature\_answer} & Stop at an intermediate value and treat it as the final result. \\
\texttt{verbose\_content\_free} & Add generic filler instead of mathematically useful content. \\
\bottomrule
\end{tabularx}
\end{table}

\begin{promptlisting}{Math flawed-trace generation: system prompt}
You will be given one math reasoning passage.

Create an internally coherent but mathematically incorrect reasoning trace.
Inject the specified mathematical reasoning error in 1-4 places in this passage.

Important requirements:
- Do not rewrite the whole passage.
- Preserve most of the original wording, order, structure, and reasoning style.
- Keep the same format as the input passage.
- Change only the part needed to inject the requested error type.
- After introducing the error, update any downstream calculations so the reasoning remains self-consistent.
- The final answer implied by the reasoning must be different from the original correct answer.
- Keep all unrelated parts unchanged whenever possible.
- Avoid isolated contradictions such as writing 7 * 3 = 22 unless the requested error type is specifically arithmetic_slip.
- Do not add meta commentary.
- Do not explain what error you introduced.
- Do not mention that the reasoning is wrong.
- Do not point out contradictions, inconsistencies, or failures in the modified reasoning.
- Write the modified passage as if it were a normal math reasoning attempt.
- Do not output code blocks.
- Do not add a separate final answer line.
- Directly output the modified reasoning passage only.
\end{promptlisting}

\begin{promptlisting}{Math flawed-trace generation: user template}
The math error type is: {display_name}
This error type means: {definition}
Target negative kind slug: {negative_kind}
Mutation target: {mutation_target}
The final answer field is provided only as context and should not be copied as a separate final answer: {action_gt}

The math reasoning passage is:
{raw_reason}
\end{promptlisting}

\begin{specbox}{Math flawed-trace parsing and filtering}
The returned text is stored as a cleaned negative reasoning trace. Group construction deduplicates repeated negatives and keeps only groups with at least four valid flawed traces. Positive GSM8K traces are cleaned before perturbation by removing final-answer markers and calculator annotations.
\end{specbox}

\subsection{Rubric Annotation Prompts}
\label{app:rubric-prompts}

The rubric judge is used only to annotate the offline training data. It may receive the final code or answer to assess whether the reasoning supports the final artifact. The trained Reason RM and the online planner reward do not receive final answers, generated code, executor outputs, or verifier outcomes as input.

\subsubsection{Code Rubric Judge}

\begin{promptlisting}{Code rubric judge prompt}
You are judging the quality of reasoning for a code task.

You will be given:
- a coding problem
- an answer-free reasoning text
- the final code / result

Score the reasoning itself instead of the final answer alone.

Rubric dimensions (0-10, where 10 is best):
1. task_understanding: Does the reasoning correctly understand the task and constraints?
2. plan_quality: Is the proposed approach sensible and appropriate?
3. step_coherence: Are the reasoning steps logically connected and internally consistent?
4. action_support: Does the reasoning actually support the final code/result?
5. non_leakage: Does the reasoning avoid simply dumping the final implementation or over-leaking the final action?

Scoring rules:
- Each dimension must be scored from 0 to 10.
- Decimal scores are allowed.
- Use at most one decimal place when needed.
- 10 means excellent reasoning quality on that dimension.
- 0 means the reasoning completely fails on that dimension.

Then produce:
- rubric_score: overall score from 0 to 10 (decimal allowed)
- rubric_label: one of [strong, acceptable, weak, bad]
- rubric_reason: short explanation in 1-3 sentences

Return STRICT JSON only, with this exact schema:
{
  "task_understanding": 0.0,
  "plan_quality": 0.0,
  "step_coherence": 0.0,
  "action_support": 0.0,
  "non_leakage": 0.0,
  "rubric_score": 0.0,
  "rubric_label": "bad",
  "rubric_reason": "..."
}

Problem:
{problem}

Reasoning:
{reason_clean}

Final code/result:
{action_gt}
\end{promptlisting}

\subsubsection{Math Rubric Judge}

\begin{promptlisting}{Math rubric judge: system prompt}
You are an expert math reasoning rubric judge. Evaluate the quality of the reasoning process for solving math problems, not just whether the final answer string is present.
\end{promptlisting}

\begin{promptlisting}{Math rubric judge: user prompt}
You are judging the quality of reasoning for a math problem.

You will be given:
- a math problem
- a mathematical reasoning text
- the final answer

Score the reasoning itself instead of the final answer alone. A reasoning trace can deserve a low score even when it states the correct answer if the derivation is unsupported, incoherent, or mathematically invalid.

Rubric dimensions (0-10, where 10 is best):
1. problem_understanding: Does the reasoning correctly identify the quantities, conditions, and goal?
2. solution_strategy: Is the mathematical approach appropriate and efficient for the problem?
3. step_coherence: Are the reasoning steps logically connected and internally consistent?
4. calculation_correctness: Are computations, algebraic manipulations, and unit conversions correct?
5. answer_support: Does the reasoning actually justify the final answer without unsupported jumps?

Scoring rules:
- Each dimension must be scored from 0 to 10.
- Decimal scores are allowed.
- Use at most one decimal place when needed.
- 10 means excellent reasoning quality on that dimension.
- 0 means the reasoning completely fails on that dimension.

Then produce:
- rubric_score: overall score from 0 to 10 (decimal allowed)
- rubric_label: one of [strong, acceptable, weak, bad]
- rubric_reason: short explanation in 1-3 sentences

Return STRICT JSON only, with this exact schema:
{
  "problem_understanding": 0.0,
  "solution_strategy": 0.0,
  "step_coherence": 0.0,
  "calculation_correctness": 0.0,
  "answer_support": 0.0,
  "rubric_score": 0.0,
  "rubric_label": "bad",
  "rubric_reason": "..."
}

Problem:
{problem}

Reasoning:
{reason_clean}

Final answer:
{action_gt}
\end{promptlisting}

\begin{table}[h]
\centering
\small
\setlength{\tabcolsep}{7pt}
\renewcommand{\arraystretch}{1.05}
\caption{Conversion from raw judge scores to Reason RM training labels. Dimension labels are used by the rubric heads, while the normalized total score is used by the aggregate head.}
\label{tab:app-score-conversion}
\begin{tabular}{cl}
\toprule
RM label & Raw judge-score interval \\
\midrule
0 & $[0,2)$ \\
1 & $[2,4)$ \\
2 & $[4,6)$ \\
3 & $[6,8)$ \\
4 & $[8,10]$ \\
\bottomrule
\end{tabular}
\end{table}

The exact conversion is
\[
y_k=\min\!\left(4,\max\!\left(0,\left\lfloor s_k/2 \right\rfloor\right)\right),
\qquad
s_{\mathrm{total}}=\mathrm{clip}(s_{\mathrm{rubric}}/10,0,1),
\]
where $s_k$ is a raw dimension score and $s_{\mathrm{rubric}}$ is the raw total rubric score.

\subsection{Group Schema and Reason RM Input Rendering}
\label{app:dataset-schema}

\begin{schemalisting}{Reasoning-group schema}
{
  "problem_id": "code_000000 | math_gsm8k_000000",
  "source": "code | gsm8k",
  "task_type": "math optional",
  "problem": "...",
  "reference_solution": "...",
  "positive_pool": [
    {
      "reasoning": "...",
      "label": "positive",
      "rubric": {
        "dimension_name": 0,
        "total": 0.0
      },
      "rubric_label": "positive | strong | acceptable | weak | bad",
      "rubric_score_raw": 10.0,
      "raw_reason_length": 0,
      "clean_reason_length": 0
    }
  ],
  "negative_bank": [
    {
      "reasoning": "...",
      "negative_kind": "...",
      "negative_index": 0,
      "label": "negative",
      "rubric": {
        "dimension_name": 0,
        "total": 0.0
      },
      "rubric_label": "bad",
      "rubric_score_raw": 0.0,
      "rubric_reason": "...",
      "raw_reason_length": 0,
      "clean_reason_length": 0,
      "metadata": {}
    }
  ],
  "metadata": {
    "source_dataset": "...",
    "source_row_index": 0,
    "negative_count": 0,
    "dimension_names": []
  }
}
\end{schemalisting}

\begin{schemalisting}{Reason RM input rendering}
<task>
...
</task>
<type>code|math</type>
<problem>
...
</problem>
<reasoning>
...
</reasoning>
\end{schemalisting}

The Reason RM input contains only the task type, problem text, and candidate reasoning. It excludes the reference solution, final answer, generated code, executor output, and verifier result. This separation prevents the Reason RM from learning a shortcut based on final-artifact correctness rather than reasoning quality.

\section{\datasetname{} Statistics and Reason RM Validation}
\label{app:rm-validation}

This appendix audits the grouped supervision data used to train the Reason RM and validates whether the trained RM learns the intended reason-quality ordering. The main paper describes the full code-and-math construction. Here we report the detailed exported audit for the GSM8K-derived math split, because this split contains complete per-perturbation metadata and all-negative held-out RM predictions. The goal is not to claim broad reward-model generalization beyond the construction distribution, but to verify that the learned RM reliably distinguishes high-quality reference reasoning from localized flawed traces under the same grouped supervision format used for training.

\subsection{Grouped Data Statistics}
\label{app:grouped-data-statistics}

Table~\ref{tab:app-math-group-stats} summarizes the retained math reasoning groups. Starting from 3,000 GSM8K seed problems, the construction keeps 2,989 groups after filtering, corresponding to a retention rate of $99.63\%$. Each retained group contains one reference trace and at least four flawed traces, with an average of $6.28$ flawed traces per problem. Lengths in this appendix are whitespace-token counts computed from the stored reasoning text.

\begin{table}[h]
\centering
\small
\setlength{\tabcolsep}{6pt}
\renewcommand{\arraystretch}{1.08}
\caption{Statistics of the audited GSM8K-derived \datasetname{} split.}
\label{tab:app-math-group-stats}
\begin{tabular}{lrrrrrr}
\toprule
Split & Seed problems & Kept groups & Reference traces & Flawed traces & Avg. ref len. & Avg. flawed len. \\
\midrule
Math & 3000 & 2989 & 2989 & 18764 & 49.95 & 78.85 \\
\bottomrule
\end{tabular}
\end{table}

The high retention rate indicates that the perturbation pipeline rarely fails to produce enough valid local negatives. The flawed traces are longer on average than the references, mainly because the perturbation set includes verbose content-free reasoning. This length asymmetry is useful for evaluation: a reward model that simply favors longer reasoning would be pulled toward some flawed candidates, whereas a valid reason-quality model should still rank compact reference reasoning above longer but less supportive traces.

\subsection{Perturbation Coverage}
\label{app:perturbation-coverage}

Table~\ref{tab:app-math-perturbation-stats} reports the perturbation distribution and average rubric score of flawed traces. The seven perturbation types are all represented at nontrivial frequency, with the largest type accounting for $15.90\%$ of negatives and the smallest accounting for $10.89\%$. The average normalized judge scores are low across all perturbation types, confirming that the perturbations reduce reasoning quality while preserving task relevance.

\begin{table}[h]
\centering
\small
\setlength{\tabcolsep}{6pt}
\renewcommand{\arraystretch}{1.08}
\caption{Perturbation statistics for flawed math reasoning traces. Judge scores are normalized to $[0,1]$.}
\label{tab:app-math-perturbation-stats}
\begin{tabular}{lrrrr}
\toprule
Perturbation type & Count & Share & Avg. judge score & Avg. length \\
\midrule
\texttt{arithmetic\_slip} & 2628 & 14.01\% & 0.36 & 49.92 \\
\texttt{dropped\_case} & 2044 & 10.89\% & 0.32 & 55.57 \\
\texttt{premature\_answer} & 2781 & 14.82\% & 0.37 & 44.17 \\
\texttt{unit\_mismatch} & 2616 & 13.94\% & 0.25 & 56.91 \\
\texttt{unsupported\_jump} & 2853 & 15.21\% & 0.32 & 49.19 \\
\texttt{verbose\_content\_free} & 2858 & 15.23\% & 0.29 & 235.44 \\
\texttt{wrong\_operation} & 2984 & 15.90\% & 0.23 & 50.22 \\
\bottomrule
\end{tabular}
\end{table}

Two patterns are important. First, the perturbation set is not dominated by a single easy flaw type: arithmetic errors, wrong operations, unit mistakes, dropped cases, unsupported jumps, premature stopping, and verbose non-supportive traces all appear frequently. This matters because the Reason RM should learn a general notion of reasoning support rather than a detector for one synthetic artifact. Second, \texttt{verbose\_content\_free} traces are by far the longest but still receive low judge scores. This directly supports the design goal of discouraging verbosity as a shortcut for reasoning quality.

\subsection{Rubric Label Distribution}
\label{app:rubric-label-distribution}

The rubric labels are produced by discretizing raw judge scores into five classes, from 0 to 4. Table~\ref{tab:app-dimension-distribution} groups these labels into low, middle, good, and strong regions. The distribution shows that the perturbations are local rather than destructive. Many flawed traces still receive high problem-understanding labels, because they often preserve the original problem setup. However, dimensions that depend on actually supporting the solution, especially calculation correctness and answer support, shift strongly toward low labels.

\begin{table}[h]
\centering
\small
\setlength{\tabcolsep}{6pt}
\renewcommand{\arraystretch}{1.08}
\caption{Rubric label distribution over all annotated math candidates. Low combines labels 0 and 1.}
\label{tab:app-dimension-distribution}
\begin{tabular}{lrrrr}
\toprule
Dimension & Low: 0--1 & Mid: 2 & Good: 3 & Strong: 4 \\
\midrule
Problem understanding & 11.5\% & 13.8\% & 36.2\% & 38.5\% \\
Solution strategy & 20.8\% & 20.4\% & 23.2\% & 35.7\% \\
Step coherence & 52.5\% & 21.8\% & 10.1\% & 15.6\% \\
Calculation correctness & 72.1\% & 8.7\% & 3.5\% & 15.7\% \\
Answer support & 83.6\% & 2.4\% & 0.2\% & 13.8\% \\
\bottomrule
\end{tabular}
\end{table}

This distribution is desirable for grouped reason-only supervision. If all flawed traces simply misunderstood the problem, the RM could solve the task by detecting superficial topic mismatch. Instead, many flawed traces still look relevant at the problem-understanding level, while failing in coherence, calculation, or answer support. The grouped comparison therefore forces the RM to attend to whether the trace actually sustains a valid solution path.

\subsection{Held-Out Reason RM Validation}
\label{app:heldout-rm-validation}

We validate the Reason RM on held-out GSM8K perturbation groups. The validation uses an all-negative protocol: for each held-out problem, the reference trace is compared against all available flawed traces in the group, rather than against a single sampled negative. This directly tests the use case needed by the online reward: the RM should assign higher scores to reliable reasoning than to plausible local perturbations.

\begin{table*}[h]
\centering
\scriptsize
\setlength{\tabcolsep}{4pt}
\renewcommand{\arraystretch}{1.08}
\caption{Held-out Reason RM validation on GSM8K perturbation groups. Pairwise accuracy measures whether the reference trace scores above a flawed trace from the same problem. Group accuracy measures whether the reference trace is the top-scoring candidate in the group.}
\label{tab:app-rm-validation}
\resizebox{\textwidth}{!}{
\begin{tabular}{lrrrrrrrrr}
\toprule
RM & Groups & Candidates & Neg. pairs & Pairwise acc. & Group acc. & Total Spearman & Total MAE & Total RMSE & Dim. acc. \\
\midrule
Math LoRA-CE step504 & 299 & 2186 & 1887 & 99.15\% & 96.32\% & 0.791 & 0.105 & 0.156 & 66.68\% \\
\bottomrule
\end{tabular}
}
\end{table*}

The RM shows strong separation between reference and flawed reasoning. Its mean score for reference traces is $0.933$, compared with $0.273$ for flawed traces, giving a mean pairwise margin of $0.659$. The high pairwise accuracy shows that the RM reliably rejects localized perturbations, while the high group accuracy shows that this ranking remains robust when all flawed candidates compete against the reference simultaneously. The score correlation and low total-score error further indicate that the RM is not only learning a binary reference-versus-flaw separator, but also preserving a useful ordering of reasoning quality.

\subsection{Validation by Perturbation Type}
\label{app:rm-perturbation-breakdown}

Table~\ref{tab:app-rm-perturbation-breakdown} breaks down pairwise ranking accuracy by perturbation type. The RM performs well across all seven flaw categories, including subtle arithmetic slips and more structural wrong-operation or unit-mismatch errors.

\begin{table}[h]
\centering
\small
\setlength{\tabcolsep}{6pt}
\renewcommand{\arraystretch}{1.08}
\caption{Held-out Reason RM ranking accuracy by perturbation type. Mean margin is the average RM score difference between the reference trace and the flawed trace.}
\label{tab:app-rm-perturbation-breakdown}
\begin{tabular}{lrrr}
\toprule
Perturbation type & Pairs & Pairwise acc. & Mean margin \\
\midrule
\texttt{arithmetic\_slip} & 271 & 97.79\% & 0.520 \\
\texttt{dropped\_case} & 195 & 98.46\% & 0.614 \\
\texttt{premature\_answer} & 281 & 99.64\% & 0.650 \\
\texttt{unit\_mismatch} & 264 & 99.62\% & 0.730 \\
\texttt{unsupported\_jump} & 286 & 99.30\% & 0.648 \\
\texttt{verbose\_content\_free} & 292 & 99.32\% & 0.697 \\
\texttt{wrong\_operation} & 298 & 99.66\% & 0.733 \\
\bottomrule
\end{tabular}
\end{table}

The lowest accuracy appears on \texttt{arithmetic\_slip}, which is expected because these perturbations can be extremely local and may leave most of the surrounding derivation intact. Even there, the RM reaches $97.79\%$ pairwise accuracy. The largest margins appear for \texttt{wrong\_operation} and \texttt{unit\_mismatch}, where the perturbation changes the mathematical structure or quantity semantics in a way that should strongly reduce executor usefulness. This breakdown is consistent with the intended role of the RM: it is sensitive not only to fluent formatting, but also to whether the reasoning trace preserves the calculation and support needed for a correct downstream artifact.

\subsection{Takeaways}
\label{app:rm-validation-takeaways}

The audit supports three conclusions. First, the grouped construction creates dense same-problem comparisons: almost every seed yields a valid group, and each group contains multiple localized flawed traces. Second, the perturbation distribution is diverse enough that the RM cannot rely on a single artifact such as length, topic mismatch, or generic incoherence. Third, the trained Reason RM accurately recovers the intended within-group ordering on held-out perturbation groups, with high pairwise and group-level ranking accuracy. These properties justify using the RM as the intrinsic reasoning-quality component in TraceLift, while the executor-uplift term in the main reward further filters this quality signal by downstream utility.

\section{Additional Ablations and Training Dynamics}
\label{app:additional-ablations-dynamics}

This appendix provides diagnostic views of the ablation and training-dynamics results. The main text reports the absolute benchmark numbers. Here we focus on deltas, reward failure modes, and temporal behavior, which better expose why each component of \method{} is needed. All code ablations in this section use the Qwen2.5-7B two-stage code setting. During training-time executor comparisons, the with-reasoning and no-reasoning branches use matched sampling settings with temperature $0.5$; final evaluation uses greedy decoding with temperature $0$.

\subsection{Uplift Rollout Count}
\label{app:k-sweep-diagnostics}

The uplift term estimates whether a candidate trace changes the frozen executor's success probability relative to a no-reasoning baseline. Table~\ref{tab:app-k-delta} reports drops relative to the main $K=3$ setting. This delta view makes clear that $K=3$ is not only best on the aggregate score, but also consistently stronger than both $K=1$ and $K=5$ across all four code benchmarks.

\begin{table}[h]
\centering
\small
\setlength{\tabcolsep}{7pt}
\renewcommand{\arraystretch}{1.08}
\caption{Effect of the number of executor comparisons used for uplift estimation. Entries are percentage-point changes relative to the main $K=3$ setting. Negative values indicate a drop from $K=3$.}
\label{tab:app-k-delta}
\begin{tabular}{lrr}
\toprule
Benchmark & $K=1$ minus $K=3$ & $K=5$ minus $K=3$ \\
\midrule
HumanEval & $-4.27$ & $-5.49$ \\
HumanEval+ & $-0.61$ & $-3.65$ \\
MBPP-full & $-0.40$ & $-0.80$ \\
LiveCodeBench & $-1.00$ & $-0.50$ \\
Micro avg. & $-1.14$ & $-1.71$ \\
\bottomrule
\end{tabular}
\end{table}

The pattern suggests a bias--variance tradeoff in the executor-utility estimate. With $K=1$, the reward is high variance: a trace can be credited because one sampled executor completion happens to pass, even when the trace is not robustly useful. Increasing to $K=3$ reduces this noise while preserving enough contrast between traces in the same GRPO group. With $K=5$, the signal becomes more conservative: small positive and negative changes are more likely to average out, which weakens the relative advantage assigned to traces that provide sharp but localized executor guidance. Since final evaluation is deterministic pass@1, the best training-time estimator is not necessarily the most averaged one; it is the one that best preserves useful contrast for pass@1 optimization.

\subsection{Reward Component Diagnostics}
\label{app:reward-component-diagnostics}

Table~\ref{tab:app-reward-drops} reports how much each ablated reward drops relative to the full \method{} reward. Unlike the main ablation table, this view isolates the magnitude and location of each failure mode.

\begin{table}[h]
\centering
\small
\setlength{\tabcolsep}{6pt}
\renewcommand{\arraystretch}{1.08}
\caption{Reward-component diagnostics. Entries are percentage-point drops relative to full \method{}.}
\label{tab:app-reward-drops}
\begin{tabular}{lrrr}
\toprule
Benchmark & No-uplift & RM-uplift only & LLM-as-judge \\
\midrule
HumanEval & $-6.71$ & $-1.22$ & $-5.49$ \\
HumanEval+ & $-3.65$ & $-0.61$ & $-1.82$ \\
MBPP-full & $-2.80$ & $-2.60$ & $-2.40$ \\
LiveCodeBench & $-2.50$ & $-2.75$ & $-4.25$ \\
Micro avg. & $-3.34$ & $-2.20$ & $-3.34$ \\
\bottomrule
\end{tabular}
\end{table}

The three ablations fail for different reasons. The No-uplift reward removes the executor-consumption check from the reasoning score. Its large drop on HumanEval indicates that a rubric-good trace is not always an executor-useful trace: it may be coherent and fluent while still omitting the exact edge condition or implementation constraint that the executor needs. The RM-uplift-only reward keeps executor grounding, but removes the verifier anchor. It is closer to full \method{} on HumanEval and HumanEval+, where relative executor improvement often aligns with pass/fail success, but it loses more on MBPP-full and LiveCodeBench, where a trace can improve the executor relative to the no-reasoning branch without making the final artifact correct. This shows that uplift is a useful marginal-utility signal, but it should not replace the task verifier. Finally, the LLM-as-judge variant underperforms because the judge score is not calibrated enough to serve as a dense on-policy reward.

Table~\ref{tab:app-reward-summary-delta} summarizes the same point at the aggregate level by comparing each reward with Exec-only and full \method{}.

\begin{table}[h]
\centering
\small
\setlength{\tabcolsep}{7pt}
\renewcommand{\arraystretch}{1.08}
\caption{Aggregate reward diagnostics on Qwen2.5-7B code. Deltas are percentage points on the code micro-average.}
\label{tab:app-reward-summary-delta}
\begin{tabular}{lrrr}
\toprule
Reward & Micro avg. & $\Delta$ vs. Exec-only & $\Delta$ vs. \method{} \\
\midrule
Exec-only & 52.28 & 0.00 & $-2.61$ \\
No-uplift & 51.55 & $-0.73$ & $-3.34$ \\
RM-uplift only & 52.69 & $+0.41$ & $-2.20$ \\
LLM-as-judge & 51.55 & $-0.73$ & $-3.34$ \\
\method{} & 54.89 & $+2.61$ & 0.00 \\
\bottomrule
\end{tabular}
\end{table}

This aggregate view supports the design of Eq.~\ref{eq:reward}. Reason-quality supervision alone is insufficient if it is detached from the executor. Executor uplift alone is insufficient if it is detached from final task success. A generic LLM judge is insufficient if its scores are saturated or poorly calibrated. The full reward combines all three requirements: task success, intrinsic reasoning quality, and measured executor utility.

\subsection{Why Direct LLM Judging Is Not Enough}
\label{app:llm-judge-diagnostics}

The LLM-as-judge replacement uses the same rubric text as the Reason RM, but queries Qwen2.5-7B-Instruct directly during reward computation. The final score is normalized as $s_{\mathrm{judge}}=\mathrm{rubric\_score}/10$ and used in
\[
\mathcal{R}(P,R)
=
0.5\,\Rexec(P,R)
+
0.5\,s_{\mathrm{judge}}(P,R)\,u_{\mathrm{exec}}(P,R).
\]
The ablation is weaker than the trained Reason RM, and the logged judge scores explain why. Table~\ref{tab:app-judge-saturation} reports the score saturation diagnostic from on-policy samples.

\begin{table}[h]
\centering
\small
\setlength{\tabcolsep}{8pt}
\renewcommand{\arraystretch}{1.08}
\caption{LLM-as-judge score saturation diagnostic from logged on-policy samples.}
\label{tab:app-judge-saturation}
\begin{tabular}{lr}
\toprule
Diagnostic & Value \\
\midrule
Logged samples & 600 \\
Samples with $s_{\mathrm{judge}}=1.0$ & 573 \\
Saturation rate & 95.50\% \\
Mean judge score & 0.990 \\
\bottomrule
\end{tabular}
\end{table}

This saturation makes the judge almost non-discriminative inside a GRPO group. When most candidate traces receive near-perfect rubric scores, the reward term $s_{\mathrm{judge}}\cdot u_{\mathrm{exec}}$ collapses toward uplift-only behavior and loses the ability to prefer traces that are cleaner, more complete, or less misleading. In contrast, the trained Reason RM is optimized on same-problem reference-versus-perturbation groups, so its score is shaped to resolve exactly the local distinctions that arise during planner training. This explains why merely prompting an instruction-tuned model with the rubric does not substitute for a calibrated Reason RM.

\subsection{LLM-as-judge replacement.}
Table~\ref{tab:judge-placeholder} replaces the trained Reason RM with a direct LLM-as-judge score from Qwen2.5-7B-Instruct under the same rubric. The reward is
\[
\mathcal{R}(P,R)
=
0.5\,\Rexec(P,R)
+
0.5\,s_{\mathrm{judge}}(P,R)\,\urm(P,R),
\]
where $s_{\mathrm{judge}}$ is the normalized rubric score returned by the judge model.

\begin{table}[t]
\centering
\small
\setlength{\tabcolsep}{7pt}
\renewcommand{\arraystretch}{1.08}
\caption{Ablation results of Qwen2.5-7B LLM-as-judge replacement. All entries are pass@1 percentages.}
\label{tab:judge-placeholder}
\begin{tabular}{lrrr}
\toprule
Benchmark & Base & \method{} & LLM-as-judge \\
\midrule
HumanEval & 68.90 & \best{72.56} & 67.07 \\
HumanEval+ & 60.37 & \best{64.02} & 62.20 \\
MBPP-full & 58.00 & \best{61.20} & 58.80 \\
LiveCodeBench & 29.50 & \best{36.00} & 31.75 \\
Micro avg. & 50.49 & \best{54.89} & 51.54 \\
\bottomrule
\end{tabular}
\end{table}

The LLM-as-judge variant is substantially weaker than the trained Reason RM, with a micro-average of $51.54\%$ compared with $54.89\%$ for \method{}. This result suggests that executor grounding alone is not sufficient if the reasoning-quality score is noisy or poorly calibrated. A generic instruction-tuned judge can follow the rubric at a surface level, but it is not trained on the grouped reference-versus-perturbation structure of \datasetname{} and may over-credit fluent or plausible traces that miss subtle algorithmic constraints. In GRPO, such score noise is amplified because rewards are compared within small on-policy groups; even when multiplied by uplift, an unstable judge score can distort which trace receives credit. The trained Reason RM provides a lower-variance and more task-calibrated signal, making it a better reward component for on-policy planner optimization than direct rubric judging.

\subsection{Training-Dynamics Diagnostics}
\label{app:training-dynamics-diagnostics}

The main text shows the full reasoning-length and uplift curves. Here we summarize the same logs with rolling-window statistics. The rolling mean window is 30 steps. The Base curve is obtained by re-running the Qwen2.5-7B base planner on the same training problem sequence, so it reflects problem-sequence variation rather than parameter updates. Exec-only and \method{} are on-policy training samples.

\begin{table}[h]
\centering
\small
\setlength{\tabcolsep}{6pt}
\renewcommand{\arraystretch}{1.08}
\caption{Rolling-window training dynamics. Token values are Qwen2.5-tokenizer counts. The last column is the rolling mean uplift coefficient used in the \method{} reward.}
\label{tab:app-rolling-windows}
\begin{tabular}{lrrrrr}
\toprule
Window & Base len. & Exec-only len. & \method{} len. & \method{} $-$ Exec-only & Uplift coeff. \\
\midrule
Steps 1--100 & 130.40 & 158.62 & 126.01 & $-32.60$ & 0.0306 \\
Steps 251--350 & 128.01 & 131.99 & 142.18 & $+10.19$ & 0.0791 \\
Steps 501--600 & 129.08 & 148.87 & 131.64 & $-17.23$ & 0.0711 \\
Final step & 128.15 & 166.02 & 136.25 & $-29.77$ & 0.0820 \\
\bottomrule
\end{tabular}
\end{table}

The dynamics show that \method{} does not follow a monotonic length-increase path. In the early window, it is shorter than both Base and Exec-only while the uplift coefficient is already positive. In the middle window, length increases temporarily as the planner explores more detailed traces, and the uplift coefficient rises. By the late window and final step, the planner retains the positive uplift signal while returning to much shorter traces than Exec-only. This pattern is consistent with a refinement process: the reward initially permits exploration of useful details, then favors traces that keep the executor-relevant information without accumulating unnecessary explanation.

The final-step contrast is especially informative. Exec-only reaches a rolling length of $166.02$ tokens, while \method{} is $136.25$ tokens, nearly $30$ tokens shorter, despite achieving better downstream performance. Thus, the additional reward term is not functioning as a verbosity bonus. It acts as a filter that rewards details only when they are both rubric-supported and executor-useful.

\section{Qualitative Case Studies}
\label{app:qualitative-cases}

This appendix provides qualitative diagnostics for representative code and math examples. The goal is not to repeat the aggregate benchmark results, but to inspect how different training objectives change the information carried by the reasoning trace. Each case is summarized from evaluation artifacts and focuses on the decisive reasoning signal rather than reproducing full model completions. Across cases, the recurring pattern is that \method{} does not merely produce more text; it tends to preserve the specific constraint, invariant, unit relation, or algebraic dependency that changes the frozen executor's final behavior.

\subsection{Code Case Studies}
\label{app:code-cases}

\begin{casebox}{HumanEval/64: \texttt{vowels\_count}}
\small
\begin{tabularx}{\linewidth}{@{}lY@{}}
\toprule
Failure type & \badtag{edge-case handling} \quad \badtag{duplicate counting} \\
\midrule
Task signal & Count ordinary vowels, with a special rule for terminal \texttt{y}. \\
Base behavior & The reasoning covers ordinary vowels but tends to omit the terminal-\texttt{y} rule. This gives the executor an underspecified plan. \\
Exec-only behavior & The reasoning expands the plan, but can place \texttt{y} inside the vowel set and also describe a separate terminal-\texttt{y} adjustment. This creates a duplicate-counting path. \\
\method{} behavior & The reasoning keeps the ordinary vowel set as \texttt{aeiou} and handles only word-final \texttt{y} as a separate boundary case. \\
\bottomrule
\end{tabularx}

\vspace{0.4em}
\textbf{Analysis.}
This case illustrates why execution-only feedback can be too coarse. A longer trace may mention the missing edge case, but if it encodes the edge case in the wrong form, the executor can implement a systematically wrong counting rule. \method{} improves the trace by separating the default rule from the exceptional rule. The key gain is therefore not verbosity, but a cleaner partition of cases: ordinary vowels are counted uniformly, and terminal \texttt{y} is treated exactly once. This is the kind of local constraint that a rubric-only score may not fully identify unless it is grounded in whether the executor's behavior actually changes.
\end{casebox}

\begin{casebox}{HumanEvalPlus/55: \texttt{fib}}
\small
\begin{tabularx}{\linewidth}{@{}lY@{}}
\toprule
Failure type & \badtag{infeasible complexity} \quad \badtag{stronger-test timeout} \\
\midrule
Task signal & Compute Fibonacci values in a way that remains feasible under stronger tests. \\
Base behavior & The reasoning can select the mathematically natural recursive definition, which is correct as a recurrence but inefficient as an implementation plan. \\
Exec-only behavior & The trace remains compatible with naive recursion and therefore may still lead the executor to redundant exponential computation. \\
\method{} behavior & The trace specifies an iterative plan, maintaining the two most recent Fibonacci values and updating them in a loop. \\
\bottomrule
\end{tabularx}

\vspace{0.4em}
\textbf{Analysis.}
This case separates mathematical correctness from executor-useful planning. The recursive definition of Fibonacci is semantically correct, so a surface-level reasoning judge may view it as plausible. However, for code generation, the reasoning trace must also communicate an executable route that satisfies the test regime. \method{} shifts the trace from a definition-level explanation to an implementation-level invariant: keep two previous values and iterate. This supports the paper's central claim that reasoning quality should be evaluated as an intermediate artifact consumed by an executor, not only as a fluent explanation of the task.
\end{casebox}

\begin{casebox}{MBPP/16: \texttt{text\_lowercase\_underscore}}
\small
\begin{tabularx}{\linewidth}{@{}lY@{}}
\toprule
Failure type & \badtag{missing full-string constraint} \quad \badtag{regex anchoring} \\
\midrule
Task signal & Check whether a string consists of lowercase letters, one underscore, and lowercase letters. \\
Base behavior & The trace tends to describe a pattern similar to \texttt{[a-z]+\_[a-z]+}, which can match a substring instead of the whole input. \\
Exec-only behavior & The reasoning remains close to the same unanchored pattern and does not reliably force full-string matching. \\
\method{} behavior & The trace explicitly specifies the anchored pattern \texttt{\textasciicircum{}[a-z]+\_[a-z]+\$}. \\
\bottomrule
\end{tabularx}

\vspace{0.4em}
\textbf{Analysis.}
The decisive difference is a small symbolic constraint. Adding anchors does not make the reasoning much longer, but it changes the executor's implementation path from substring search to full-string validation. This is a useful example of executor-consumable reasoning: a compact trace carries exactly the information needed to avoid a common false-positive pattern. It also explains why length alone is a poor proxy for reasoning quality; the important unit of supervision is whether the trace contains the constraint that affects execution.
\end{casebox}

\begin{casebox}{HumanEval/95: \texttt{check\_dict\_case}}
\small
\begin{tabularx}{\linewidth}{@{}lY@{}}
\toprule
Failure type & \badtag{missing type guard} \quad \badtag{invalid-key behavior} \\
\midrule
Task signal & Determine whether all dictionary keys are consistently lowercase or uppercase, while handling non-string keys correctly. \\
Base behavior & The trace focuses on checking uppercase/lowercase consistency but can omit the requirement that all keys must first be strings. \\
Exec-only behavior & The trace may become more compact, but the decisive improvement depends on whether it explicitly includes the non-string guard. \\
\method{} behavior & The trace states that every key must be a string before applying lower/upper consistency checks. \\
\bottomrule
\end{tabularx}

\vspace{0.4em}
\textbf{Analysis.}
This case demonstrates the value of type-level constraints. Without the string-key guard, the executor may write a solution that works for ordinary dictionaries but fails on adversarial or edge inputs. \method{} improves the trace by making the guard a prerequisite rather than an afterthought. The reasoning is not valuable because it explains the task at length; it is valuable because it orders the checks correctly: validate key types first, then check case consistency.
\end{casebox}

\subsection{Math Case Studies}
\label{app:math-cases}

\begin{casebox}{GSM8K \texttt{gsm8k\_test\_0038}: weekly running hours}
\small
\begin{tabularx}{\linewidth}{@{}lY@{}}
\toprule
Failure type & \badtag{unit aggregation} \quad \badtag{rate denominator error} \\
\midrule
Task signal & Aggregate the total weekly running time before dividing total distance by time. \\
Base behavior & The executor predicts a value corresponding to an incorrect denominator, indicating that the weekly hours are not aggregated correctly. \\
Exec-only behavior & The reasoning still fails to supply the correct total-time structure and can lead to another wrong denominator. \\
\method{} behavior & The trace computes the weekly total as $3 + 1.5 + 1.5 = 6$ hours and then evaluates $60/6=10$. \\
\bottomrule
\end{tabularx}

\vspace{0.4em}
\textbf{Analysis.}
The important reasoning signal is the construction of the denominator. The problem is not solved by recognizing a generic distance-rate-time template; the executor needs the correct unit aggregation before applying the division. \method{} provides this intermediate quantity explicitly, which makes the final operation unambiguous. This is a typical executor-grounding benefit: the trace does not merely state a formula, but supplies the quantity that the executor would otherwise infer incorrectly.
\end{casebox}

\begin{casebox}{GSM-Hard \texttt{gsm\_hard\_0065}: total pets}
\small
\begin{tabularx}{\linewidth}{@{}lY@{}}
\toprule
Failure type & \badtag{dropped entity} \quad \badtag{large-number stress} \\
\midrule
Task signal & Compute the number of pets for Cindy, Marcia, and Jan, then sum all three quantities. \\
Base behavior & The trace can omit one participant from the final aggregation, producing a total that excludes Cindy. \\
Exec-only behavior & The same dropped-entity pattern remains possible under execution-only training. \\
\method{} behavior & The trace separately computes Cindy, Marcia, and Jan, and then sums the three values to obtain the final total. \\
\bottomrule
\end{tabularx}

\vspace{0.4em}
\textbf{Analysis.}
This case shows a failure that final-answer reward alone may not localize: the arithmetic can be internally consistent after one entity has been dropped. The trace appears procedural, but it no longer supports the full problem. \method{} improves the reasoning by preserving the entity inventory throughout the derivation. This is exactly the kind of support relation targeted by the Reason RM rubric: a trace should justify the answer with all required quantities, not merely perform plausible arithmetic on a subset of them.
\end{casebox}

\begin{casebox}{GSM-Hard \texttt{gsm\_hard\_0225}: bakery afternoon sales}
\small
\begin{tabularx}{\linewidth}{@{}lY@{}}
\toprule
Failure type & \badtag{dropped condition} \quad \badtag{intermediate-state tracking} \\
\midrule
Task signal & Track morning sales, remaining loaves, and the specified afternoon share. \\
Base behavior & The reasoning can jump to an answer after using only part of the sales process. \\
Exec-only behavior & The executor may still confuse the leftover quantity with the afternoon-sold quantity. \\
\method{} behavior & The trace maintains the intermediate state: first compute what remains after morning sales, then apply the afternoon fraction to the remaining loaves. \\
\bottomrule
\end{tabularx}

\vspace{0.4em}
\textbf{Analysis.}
This example highlights why step support matters. The final answer depends on applying an operation to the correct intermediate state, not merely on choosing the right operation somewhere in the solution. \method{} makes the state transition explicit. This helps the executor avoid a common shortcut: treating the remaining amount and the sold amount as interchangeable. The trace is therefore useful because it constrains the order and target of operations.
\end{casebox}

\begin{casebox}{SVAMP \texttt{svamp\_test\_0034}: Jake's balloons}
\small
\begin{tabularx}{\linewidth}{@{}lY@{}}
\toprule
Failure type & \badtag{target-entity scope} \quad \badtag{incomplete accumulation} \\
\midrule
Task signal & Combine Jake's original balloons with the additional balloons bought later. \\
Base behavior & The trace may fail to preserve the accumulation target across the story. \\
Exec-only behavior & The executor answers using only the newly bought balloons, focusing on the local event rather than Jake's total. \\
\method{} behavior & The trace keeps Jake as the target entity and adds the original $6$ balloons to the $3$ bought later, yielding $9$. \\
\bottomrule
\end{tabularx}

\vspace{0.4em}
\textbf{Analysis.}
This is a scope-tracking failure. The local action ``bought 3 balloons'' is salient, but the question asks for Jake's total. \method{} improves the trace by anchoring the computation to the target entity and maintaining the cumulative quantity. The example is useful because the corrected reasoning is simple and compact; the gain comes from selecting the right referent and preserving it across the calculation.
\end{casebox}

\begin{casebox}{MATH500 \texttt{math500\_test\_0058}: evaluating $99^2+99+1$}
\small
\begin{tabularx}{\linewidth}{@{}lY@{}}
\toprule
Failure type & \badtag{arithmetic shortcut} \quad \badtag{incorrect expansion} \\
\midrule
Task signal & Evaluate the expression exactly without rounding or approximation. \\
Base behavior & The executor rounds or approximates the square, leading to an answer near $10000$. \\
Exec-only behavior & The trace still permits an arithmetic slip, producing $9902$. \\
\method{} behavior & The trace expands $(100-1)^2=9801$, then computes $9801+99+1=9901$. \\
\bottomrule
\end{tabularx}

\vspace{0.4em}
\textbf{Analysis.}
This case is small but diagnostic. The executor does not need a long derivation; it needs the exact decomposition that prevents approximation. \method{} supplies a stable arithmetic route by rewriting $99^2$ around $100$. The improvement illustrates a general math pattern: useful reasoning often consists of a compact intermediate representation that removes ambiguity from the final computation.
\end{casebox}

\begin{casebox}{MATH500 \texttt{math500\_test\_0131}: polynomial symmetry}
\small
\begin{tabularx}{\linewidth}{@{}lY@{}}
\toprule
Failure type & \badtag{false symmetry} \quad \badtag{lost odd term} \\
\midrule
Task signal & Account for both even and odd components of the polynomial. \\
Base behavior & The reasoning can treat the polynomial as if it were even, leading to a symmetric conclusion. \\
Exec-only behavior & The same shortcut remains: the executor predicts the value implied by false symmetry. \\
\method{} behavior & The trace separates the odd linear term from the even components and derives $f(3)=8$. \\
\bottomrule
\end{tabularx}

\vspace{0.4em}
\textbf{Analysis.}
The key error is structural rather than numeric. Treating the polynomial as even discards the contribution of the odd term, so subsequent arithmetic can be consistent but based on the wrong invariant. \method{} prevents this by making the asymmetry explicit. This is a strong example for executor-grounded supervision: a fluent explanation of symmetry is actively harmful if the executor consumes it as a plan, whereas a short trace that identifies the non-even component changes the solution path.
\end{casebox}

\subsection{Cross-Case Patterns}
\label{app:qualitative-patterns}

Table~\ref{tab:app-case-patterns} summarizes the mechanism-level patterns across the case studies. These patterns are more informative than raw length or fluency because they identify what the executor needs from the trace.

\begin{table}[h]
\centering
\small
\setlength{\tabcolsep}{5pt}
\renewcommand{\arraystretch}{1.08}
\caption{Mechanism-level patterns observed in qualitative cases.}
\label{tab:app-case-patterns}
\begin{tabularx}{\linewidth}{lY Y}
\toprule
Pattern & Typical failure under weaker traces & \method{} trace behavior \\
\midrule
Boundary condition & The trace mentions the common case but misses or misrepresents the exception. & Keeps the default rule and exception as separate executor-actionable cases. \\
Type or format guard & The trace describes the main operation but omits validation constraints. & Places validation before the main operation, guiding the executor's control flow. \\
Complexity constraint & The trace gives a mathematically valid but computationally infeasible plan. & Converts the idea into an efficient implementation invariant. \\
Entity or unit tracking & The trace uses the right operation on the wrong quantity or subset. & Preserves the target entity, unit, or intermediate state until the final operation. \\
Algebraic structure & The trace applies a tempting shortcut that loses an asymmetric term or exact value. & Names the structural dependency that determines the final answer. \\
\bottomrule
\end{tabularx}
\end{table}

The case studies support three qualitative conclusions. First, \method{} traces are often valuable because they remove a specific ambiguity that the executor would otherwise resolve incorrectly. Second, the improvements are local and actionable: a regex anchor, a string-key guard, an iterative invariant, a unit denominator, or an odd polynomial term. Third, execution-only training can improve final accuracy but does not explicitly distinguish useful detail from misleading or redundant detail. The executor-grounded reward addresses this gap by crediting reasoning that is both rubric-supported and empirically useful to the frozen executor.

\section{Theoretical Analysis}
\label{app:theory}

This appendix analyzes the reward used by \method{} in the fixed planner-executor setting. The goal is not to prove global convergence of GRPO or to claim that the training-time stochastic executor distribution exactly equals the deterministic evaluation distribution. Instead, we analyze the credit-assignment structure induced by Eq.~\ref{eq:reward}. The analysis shows that \method{} can be interpreted as a quality-weighted conditional executor-utility objective: the verifier term anchors training to task success, while the Reason RM term modulates the marginal utility of the reasoning trace for the frozen executor.

\subsection{Formal Setting}
\label{app:theory-setting}

Let $P\sim\Dist$ be a problem and let a planner policy sample a reasoning trace $R\sim\pi_{\theta}(\cdot\mid P)$. The executor is frozen. For a decoding temperature $\tau$, denote by
\begin{equation}
A_{\tau}(P,R;\xi)
\end{equation}
the final artifact produced by the executor from problem $P$ and reasoning trace $R$ under executor randomness $\xi$. The no-reasoning input is denoted by $\rempty$. The verifier is deterministic given the artifact:
\begin{equation}
V(A)\in\{0,1\}.
\end{equation}
For code, $V$ is executable-test success; for math, $V$ is answer matching. Define the executor success probability under the training-time executor distribution as
\begin{equation}
\qtr(P,R)
=
\Prob_{\xi}\!\left[V(A_{\tau_{\mathrm{tr}}}(P,R;\xi))=1\right],
\label{eq:app-qtr}
\end{equation}
where $\tau_{\mathrm{tr}}=0.5$ in our training-time rollout and uplift estimation. The no-reasoning success probability is
\begin{equation}
\qtr(P,\rempty)
=
\Prob_{\xi}\!\left[V(A_{\tau_{\mathrm{tr}}}(P,\rempty;\xi))=1\right].
\label{eq:app-qtr-empty}
\end{equation}
The training-time executor uplift of a trace is
\begin{equation}
\utr(P,R)
=
\qtr(P,R)-\qtr(P,\rempty).
\label{eq:app-utr}
\end{equation}
Since both terms are probabilities, $\utr(P,R)\in[-1,1]$.

Let
\begin{equation}
\mrm(P,R)=\mathrm{RM}_{\phi}(P,R)\in[0,1]
\label{eq:app-rm-score}
\end{equation}
be the Reason RM score. The RM is fixed during planner optimization and receives only the problem and reasoning trace, not executor outputs or verifier outcomes.

\begin{appassumption}[Fixed-executor reward analysis]
\label{ass:fixed-executor}
The following analysis assumes:
\begin{enumerate}
    \item the executor parameters and executor prompt are fixed during planner training;
    \item with-reasoning and no-reasoning executor branches use matched training-time decoding settings;
    \item verifier outputs are binary and deterministic given the final artifact;
    \item the Reason RM score lies in $[0,1]$ and is fixed during each planner update;
    \item the reward analyzed is Eq.~\ref{eq:reward}:
    \begin{equation}
    \mathcal{R}(P,R)
    =
    0.5\,\Rexec(P,R)
    +
    0.5\,\mathrm{RM}_{\phi}(P,R)\,\utr(P,R).
    \label{eq:app-tracelift-reward-assumption}
    \end{equation}
\end{enumerate}
The analysis is therefore about the training-time reward distribution. Final evaluation uses temperature $0$ greedy decoding and is assessed empirically in the main experiments.
\end{appassumption}

\subsection{Executor Uplift as a Conditional Treatment Effect}
\label{app:uplift-cate}

The reasoning trace can be viewed as a treatment applied to a frozen executor. For the same problem $P$ and the same executor, we compare the potential outcome with trace $R$ against the potential outcome with the reasoning field omitted. Under matched decoding settings, the only intervention is whether the executor receives the trace.

\begin{appdefinition}[Executor treatment effect]
\label{def:executor-treatment-effect}
For a fixed problem $P$ and trace $R$, the training-time conditional executor treatment effect is
\begin{equation}
\utr(P,R)
=
\qtr(P,R)-\qtr(P,\rempty).
\label{eq:app-treatment-effect}
\end{equation}
A positive value means that the trace improves the frozen executor's success probability relative to the no-reasoning prompt. A negative value means that the trace harms the executor.
\end{appdefinition}

During training, \method{} estimates this quantity by repeated executor samples:
\begin{equation}
\widehat{u}_K(P,R)
=
\frac{1}{K}\sum_{k=1}^{K}Y^{R}_{k}
-
\frac{1}{K}\sum_{k=1}^{K}Y^{0}_{k},
\label{eq:app-uplift-estimator}
\end{equation}
where
\begin{equation}
\begin{aligned}
Y^{R}_{k}
&=
V(A_{\tau_{\mathrm{tr}}}(P,R;\xi^{R}_{k})),\\
Y^{0}_{k}
&=
V(A_{\tau_{\mathrm{tr}}}(P,\rempty;\xi^{0}_{k})).
\end{aligned}
\label{eq:app-uplift-samples}
\end{equation}

\begin{appproposition}[Unbiasedness and variance of the uplift estimator]
\label{prop:uplift-estimator}
Assume that $\{Y^{R}_{k}\}_{k=1}^{K}$ are i.i.d. Bernoulli with mean $\qtr(P,R)$ and $\{Y^{0}_{k}\}_{k=1}^{K}$ are i.i.d. Bernoulli with mean $\qtr(P,\rempty)$. If the two branches are sampled independently, then
\begin{equation}
\Expect[\widehat{u}_K(P,R)]
=
\utr(P,R),
\label{eq:app-uplift-unbiased}
\end{equation}
and
\begin{equation}
\Var[\widehat{u}_K(P,R)]
=
\frac{
\qtr(P,R)(1-\qtr(P,R))
+
\qtr(P,\rempty)(1-\qtr(P,\rempty))
}{K}
\le
\frac{1}{2K}.
\label{eq:app-uplift-variance}
\end{equation}
\end{appproposition}

\begin{proof}
By linearity of expectation,
\begin{equation}
\begin{aligned}
\Expect[\widehat{u}_K(P,R)]
&=
\frac{1}{K}\sum_{k=1}^{K}\Expect[Y^{R}_{k}]
-
\frac{1}{K}\sum_{k=1}^{K}\Expect[Y^{0}_{k}]\\
&=
\qtr(P,R)-\qtr(P,\rempty)\\
&=
\utr(P,R).
\end{aligned}
\end{equation}
For independent branches,
\begin{equation}
\Var[\widehat{u}_K(P,R)]
=
\Var\!\left[\frac{1}{K}\sum_{k=1}^{K}Y^{R}_{k}\right]
+
\Var\!\left[\frac{1}{K}\sum_{k=1}^{K}Y^{0}_{k}\right].
\end{equation}
Since the samples are i.i.d.,
\begin{equation}
\Var\!\left[\frac{1}{K}\sum_{k=1}^{K}Y^{R}_{k}\right]
=
\frac{1}{K}\qtr(P,R)(1-\qtr(P,R)),
\end{equation}
and similarly for the no-reasoning branch. Therefore,
\begin{equation}
\Var[\widehat{u}_K(P,R)]
=
\frac{
\qtr(P,R)(1-\qtr(P,R))
+
\qtr(P,\rempty)(1-\qtr(P,\rempty))
}{K}.
\end{equation}
For any Bernoulli mean $p\in[0,1]$, $p(1-p)\le 1/4$, so
\begin{equation}
\Var[\widehat{u}_K(P,R)]
\le
\frac{1/4+1/4}{K}
=
\frac{1}{2K}.
\end{equation}
\end{proof}

\begin{appremark}[Effect of clipping]
Because both empirical success rates lie in $[0,1]$, their difference lies in $[-1,1]$. Thus, for binary verifiers and success-rate estimates, the clipping in Eq.~\ref{eq:uplift} is inactive mathematically. It is an implementation safeguard and does not change the estimator under the assumptions above.
\end{appremark}

\subsection{\method{} as a Quality-Weighted Uplift Objective}
\label{app:quality-weighted-uplift}

Let the single execution score used in the reward be
\begin{equation}
X(P,R)=V(A_{\tau_{\mathrm{tr}}}(P,R;\xi)),
\label{eq:app-single-exec-score}
\end{equation}
so that
\begin{equation}
\Expect[X(P,R)\mid P,R]=\qtr(P,R).
\label{eq:app-single-exec-expectation}
\end{equation}
The population version of the \method{} reward is obtained by taking expectation over executor randomness.

\begin{appproposition}[Expected \method{} reward]
\label{prop:expected-tracelift-reward}
Under Assumption~\ref{ass:fixed-executor},
\begin{equation}
\Expect[\mathcal{R}(P,R)\mid P,R]
=
0.5\,\qtr(P,\rempty)
+
0.5\,(1+\mrm(P,R))\,\utr(P,R).
\label{eq:app-expected-tracelift-reward}
\end{equation}
\end{appproposition}

\begin{proof}
By Eq.~\ref{eq:reward},
\begin{equation}
\mathcal{R}(P,R)
=
0.5\,X(P,R)
+
0.5\,\mrm(P,R)\,\widehat{u}_K(P,R).
\end{equation}
Taking conditional expectation and using Proposition~\ref{prop:uplift-estimator},
\begin{equation}
\Expect[\mathcal{R}(P,R)\mid P,R]
=
0.5\,\qtr(P,R)
+
0.5\,\mrm(P,R)\,\utr(P,R).
\end{equation}
Since
\begin{equation}
\qtr(P,R)
=
\qtr(P,\rempty)+\utr(P,R),
\end{equation}
we substitute:
\begin{equation}
\begin{aligned}
\Expect[\mathcal{R}(P,R)\mid P,R]
&=
0.5\,\big(\qtr(P,\rempty)+\utr(P,R)\big)
+
0.5\,\mrm(P,R)\,\utr(P,R)\\
&=
0.5\,\qtr(P,\rempty)
+
0.5\,(1+\mrm(P,R))\,\utr(P,R).
\end{aligned}
\end{equation}
\end{proof}

This identity is the central theoretical interpretation of \method{}. The no-reasoning executor success $\qtr(P,\rempty)$ is a problem-level constant. The trace-dependent part is
\begin{equation}
0.5\,(1+\mrm(P,R))\,\utr(P,R).
\label{eq:app-trace-dependent-score}
\end{equation}
Thus, \method{} does not add an arbitrary preference bonus. It modulates the marginal executor utility of the trace by a bounded reasoning-quality score.

\subsection{Effect Under GRPO Group Normalization}
\label{app:grpo-normalization-theory}

GRPO normalizes rewards within a group of traces sampled for the same problem. Therefore, problem-level constants do not affect the relative advantage.

\begin{applemma}[Problem-level constants cancel under group normalization]
\label{lem:constant-cancel}
For a fixed problem $P$, let a group of rewards be $r_i=c(P)+s_i$ for $i=1,\ldots,G$, where $c(P)$ is independent of $i$. Let
\begin{equation}
\hat A_i(r)
=
\frac{r_i-\mathrm{mean}_j(r_j)}
{\mathrm{std}_j(r_j)+\delta}.
\label{eq:app-group-advantage}
\end{equation}
Then
\begin{equation}
\hat A_i(r)
=
\hat A_i(s).
\label{eq:app-constant-cancel-claim}
\end{equation}
\end{applemma}

\begin{proof}
The group mean satisfies
\begin{equation}
\begin{aligned}
\mathrm{mean}_j(r_j)
&=
\mathrm{mean}_j(c(P)+s_j)\\
&=
c(P)+\mathrm{mean}_j(s_j).
\end{aligned}
\end{equation}
Thus,
\begin{equation}
\begin{aligned}
r_i-\mathrm{mean}_j(r_j)
&=
c(P)+s_i-\big(c(P)+\mathrm{mean}_j(s_j)\big)\\
&=
s_i-\mathrm{mean}_j(s_j).
\end{aligned}
\end{equation}
The standard deviation is unchanged by adding a constant:
\begin{equation}
\mathrm{std}_j(r_j)=\mathrm{std}_j(s_j).
\end{equation}
Substituting into the definition of $\hat A_i$ gives the result.
\end{proof}

Applying Lemma~\ref{lem:constant-cancel} to Proposition~\ref{prop:expected-tracelift-reward}, the term $0.5\,\qtr(P,\rempty)$ does not affect the relative GRPO advantage within the same problem. The expected trace-specific score is therefore equivalent, for group-relative credit assignment, to
\begin{equation}
S_{\mathrm{TL}}(P,R)
=
(1+\mrm(P,R))\,\utr(P,R),
\label{eq:app-tl-trace-score}
\end{equation}
up to the positive constant factor $0.5$.

\begin{appcorollary}[Quality-gated executor utility]
\label{cor:quality-gated-uplift}
For a fixed problem $P$, the expected \method{} advantage is driven by a quality-weighted executor utility term:
\begin{equation}
(1+\mrm(P,R))\,\utr(P,R).
\label{eq:app-quality-gated-term}
\end{equation}
Moreover,
\begin{equation}
\frac{\partial}{\partial \mrm}
\Expect[\mathcal{R}(P,R)\mid P,R]
=
0.5\,\utr(P,R).
\label{eq:app-rm-derivative}
\end{equation}
Therefore:
\begin{enumerate}
    \item if $\utr(P,R)>0$, a higher RM score increases the trace's expected reward;
    \item if $\utr(P,R)=0$, the RM score alone creates no trace-specific expected reward;
    \item if $\utr(P,R)<0$, a higher RM score decreases the trace's expected reward.
\end{enumerate}
\end{appcorollary}

\begin{proof}
The derivative follows directly from Proposition~\ref{prop:expected-tracelift-reward}:
\begin{equation}
\Expect[\mathcal{R}(P,R)\mid P,R]
=
0.5\,\qtr(P,\rempty)
+
0.5\,(1+\mrm(P,R))\,\utr(P,R).
\end{equation}
Holding $P$ and $R$ fixed except for the RM score,
\begin{equation}
\frac{\partial}{\partial \mrm}
\Expect[\mathcal{R}(P,R)\mid P,R]
=
0.5\,\utr(P,R).
\end{equation}
The three cases follow from the sign of $\utr(P,R)$.
\end{proof}

This corollary formalizes the intended behavior. A fluent trace with no measured effect on the executor does not receive extra trace-specific credit from the RM score. A harmful trace receives negative uplift, so multiplying by the RM score does not hide the harm; it strengthens the penalty for traces that look high-quality but actually mislead the executor.

\subsection{Pairwise Credit Assignment}
\label{app:pairwise-credit}

Consider two candidate traces $R_a$ and $R_b$ for the same problem $P$. Let
\begin{equation}
u_a=\utr(P,R_a),
\qquad
u_b=\utr(P,R_b),
\label{eq:app-pairwise-u}
\end{equation}
and
\begin{equation}
m_a=\mrm(P,R_a),
\qquad
m_b=\mrm(P,R_b).
\label{eq:app-pairwise-m}
\end{equation}
By Proposition~\ref{prop:expected-tracelift-reward} and Lemma~\ref{lem:constant-cancel}, the expected trace-specific reward difference is
\begin{equation}
\Delta_{\mathrm{TL}}
=
0.5\left[(1+m_a)u_a-(1+m_b)u_b\right].
\label{eq:app-delta-tl}
\end{equation}
For execution-only training, the expected trace-specific difference is
\begin{equation}
\Delta_{\mathrm{exec}}
=
u_a-u_b,
\label{eq:app-delta-exec}
\end{equation}
because the no-reasoning baseline $\qtr(P,\rempty)$ is constant for the problem.

\begin{appproposition}[Preference under aligned quality and utility]
\label{prop:aligned-pairwise}
If
\begin{equation}
u_a\ge u_b\ge 0
\qquad\text{and}\qquad
m_a\ge m_b,
\label{eq:app-aligned-condition}
\end{equation}
then
\begin{equation}
\Delta_{\mathrm{TL}}\ge 0.
\label{eq:app-aligned-claim}
\end{equation}
If at least one of the inequalities is strict and either $u_b>0$ or $u_a>u_b$, then $\Delta_{\mathrm{TL}}>0$.
\end{appproposition}

\begin{proof}
We expand:
\begin{equation}
(1+m_a)u_a-(1+m_b)u_b
=
u_a-u_b+m_a u_a-m_b u_b.
\end{equation}
Rewrite the second difference:
\begin{equation}
m_a u_a-m_b u_b
=
m_a(u_a-u_b)+(m_a-m_b)u_b.
\end{equation}
Therefore,
\begin{equation}
\begin{aligned}
(1+m_a)u_a-(1+m_b)u_b
&=
(u_a-u_b)+m_a(u_a-u_b)+(m_a-m_b)u_b.
\end{aligned}
\end{equation}
Under the assumptions $u_a\ge u_b\ge 0$ and $m_a\ge m_b$, all three terms are nonnegative. Thus the whole expression is nonnegative, and so $\Delta_{\mathrm{TL}}\ge0$. If $u_a>u_b$, then the first term is positive. If $m_a>m_b$ and $u_b>0$, then the third term is positive. Hence the strict case follows.
\end{proof}

\begin{appcorollary}[Resolution when outcome utility is tied]
\label{cor:tied-uplift}
If $u_a=u_b=u>0$, then
\begin{equation}
\Delta_{\mathrm{TL}}
=
0.5\,u\,(m_a-m_b).
\label{eq:app-tied-uplift}
\end{equation}
Thus, among traces with the same positive executor uplift, \method{} prefers the one with higher reasoning-quality score.
\end{appcorollary}

\begin{proof}
Substitute $u_a=u_b=u$:
\begin{equation}
\begin{aligned}
\Delta_{\mathrm{TL}}
&=
0.5\left[(1+m_a)u-(1+m_b)u\right]\\
&=
0.5\,u\,(m_a-m_b).
\end{aligned}
\end{equation}
\end{proof}

\begin{appcorollary}[Separation of useful and harmful traces]
\label{cor:positive-negative-uplift}
If $u_a>0$ and $u_b\le 0$, then
\begin{equation}
\Delta_{\mathrm{TL}}>0
\label{eq:app-positive-negative-claim}
\end{equation}
for all $m_a,m_b\in[0,1]$.
\end{appcorollary}

\begin{proof}
Since $m_a\in[0,1]$, $(1+m_a)u_a>0$. Since $m_b\in[0,1]$ and $u_b\le0$, $(1+m_b)u_b\le0$. Hence
\begin{equation}
(1+m_a)u_a-(1+m_b)u_b>0,
\end{equation}
and therefore $\Delta_{\mathrm{TL}}>0$.
\end{proof}

These pairwise results show what the product term contributes. It does not replace executor utility. Instead, it refines executor utility by preferring traces that are both useful and reason-quality aligned, while still separating positive-uplift traces from harmful traces.

\subsection{Why the Ablated Rewards Are Weaker}
\label{app:ablation-theory}

The main ablations can be understood by comparing the population reward decompositions.

\subsubsection{No-uplift reward}

The No-uplift ablation uses
\begin{equation}
\mathcal{R}_{\mathrm{no\text{-}uplift}}(P,R)
=
0.5\,X(P,R)+0.5\,\mrm(P,R).
\label{eq:app-no-uplift-reward}
\end{equation}
Taking expectation,
\begin{equation}
\Expect[\mathcal{R}_{\mathrm{no\text{-}uplift}}(P,R)\mid P,R]
=
0.5\,\qtr(P,\rempty)
+
0.5\left(\utr(P,R)+\mrm(P,R)\right).
\label{eq:app-no-uplift-expected}
\end{equation}
After removing the problem-level constant, the trace-specific score is
\begin{equation}
S_{\mathrm{no\text{-}uplift}}(P,R)
=
\utr(P,R)+\mrm(P,R).
\label{eq:app-no-uplift-score}
\end{equation}

\begin{appproposition}[No-uplift can reward executor-useless traces]
\label{prop:no-uplift-failure}
For two traces $R_a,R_b$ for the same problem, the No-uplift score prefers $R_a$ over $R_b$ whenever
\begin{equation}
m_a-m_b>u_b-u_a.
\label{eq:app-no-uplift-condition}
\end{equation}
Therefore, a trace with lower executor uplift can be preferred solely because it has a higher RM score.
\end{appproposition}

\begin{proof}
No-uplift prefers $R_a$ over $R_b$ when
\begin{equation}
u_a+m_a>u_b+m_b.
\end{equation}
Rearranging gives
\begin{equation}
m_a-m_b>u_b-u_a.
\end{equation}
This condition can hold even when $u_a<u_b$, so the reward can prefer a lower-utility trace if its RM score is sufficiently higher.
\end{proof}

This formalizes the observed weakness of removing uplift. The RM score is no longer grounded in whether the frozen executor actually benefits from the trace. A fluent or rubric-plausible trace can receive credit even when its executor treatment effect is zero or negative.

\subsubsection{RM-uplift-only reward}

The RM-uplift-only ablation uses
\begin{equation}
\mathcal{R}_{\mathrm{RM\text{-}uplift}}(P,R)
=
\mrm(P,R)\,\utr(P,R).
\label{eq:app-rm-uplift-only-reward}
\end{equation}
This preserves executor grounding, but removes the direct verifier anchor. The trace-specific score is
\begin{equation}
S_{\mathrm{RM\text{-}uplift}}(P,R)
=
\mrm(P,R)\,\utr(P,R).
\label{eq:app-rm-uplift-only-score}
\end{equation}

\begin{appproposition}[Verifier-anchor robustness of the full reward]
\label{prop:verifier-anchor}
For the full \method{} reward, the derivative of the expected trace-specific score with respect to executor uplift is
\begin{equation}
\frac{\partial}{\partial \utr}
\left[
0.5(1+\mrm)\utr
\right]
=
0.5(1+\mrm)
\in[0.5,1].
\label{eq:app-full-uplift-derivative}
\end{equation}
For RM-uplift-only, the corresponding derivative is
\begin{equation}
\frac{\partial}{\partial \utr}
\left[
\mrm\utr
\right]
=
\mrm
\in[0,1].
\label{eq:app-rm-uplift-derivative}
\end{equation}
Thus, the full reward always retains at least half-strength direct sensitivity to executor utility, even when the RM score is underestimated, whereas RM-uplift-only can suppress the utility signal entirely when $\mrm$ is near zero.
\end{appproposition}

\begin{proof}
Both derivatives follow directly from the displayed scores. Since $\mrm\in[0,1]$,
\begin{equation}
0.5(1+\mrm)\in[0.5,1],
\end{equation}
while
\begin{equation}
\mrm\in[0,1].
\end{equation}
If $\mrm=0$, the RM-uplift-only derivative is zero, whereas the full reward derivative is $0.5$.
\end{proof}

This explains why the verifier term is not redundant. The RM-uplift product is useful as a reasoning-quality filter, but the direct verifier term ensures that training remains anchored to final task success even when the RM score is imperfect.

\subsubsection{LLM-as-judge replacement}

A direct LLM-as-judge replacement uses the same form as \method{} but substitutes a prompted judge score $s_{\mathrm{judge}}(P,R)$ for the trained RM score. The ablation is weak when the judge score saturates, because a nearly constant score cannot supply within-group quality information.

\begin{applemma}[Effect of a constant judge score]
\label{lem:constant-judge}
Suppose that, for all traces in a GRPO group for the same problem $P$, the judge score is constant:
\begin{equation}
s_{\mathrm{judge}}(P,R_i)=c
\qquad\text{for all }i,
\label{eq:app-constant-judge-condition}
\end{equation}
where $c\in[0,1]$. Then the expected judge-based reward is a positive affine transformation of the execution success probability $\qtr(P,R_i)$ within that group:
\begin{equation}
\Expect[\mathcal{R}_{\mathrm{judge}}(P,R_i)\mid P,R_i]
=
0.5(1+c)\,\qtr(P,R_i)-0.5c\,\qtr(P,\rempty).
\label{eq:app-constant-judge-reward}
\end{equation}
Consequently, the judge score contributes no additional trace-quality ranking inside the group.
\end{applemma}

\begin{proof}
The judge-based reward expectation is
\begin{equation}
0.5\,\qtr(P,R_i)
+
0.5\,c\,\big(\qtr(P,R_i)-\qtr(P,\rempty)\big).
\end{equation}
Expanding,
\begin{equation}
\begin{aligned}
0.5\,\qtr(P,R_i)
+
0.5c\,\qtr(P,R_i)
-
0.5c\,\qtr(P,\rempty)
&=
0.5(1+c)\,\qtr(P,R_i)-0.5c\,\qtr(P,\rempty).
\end{aligned}
\end{equation}
For fixed $P$, the second term is constant across $i$, and $0.5(1+c)>0$. Therefore the reward ranks traces exactly as $\qtr(P,R_i)$ ranks them and adds no independent quality signal.
\end{proof}

\begin{appcorollary}[Saturated judges reduce to execution-only ranking]
\label{cor:saturated-judge}
If a prompted judge returns nearly the same score for most on-policy traces, then the judge-based reward approaches a positively scaled execution-only reward under GRPO group normalization. It cannot reliably distinguish fluent but flawed traces from genuinely executor-useful reasoning.
\end{appcorollary}

\begin{proof}
If $s_{\mathrm{judge}}(P,R_i)=c$ exactly, Lemma~\ref{lem:constant-judge} gives a positive affine transformation of $\qtr(P,R_i)$ within the group. Adding a group-level constant has no effect by Lemma~\ref{lem:constant-cancel}. Positive scaling changes the magnitude of normalized advantages by a group-wide factor but does not introduce any new ranking signal. If the judge scores are nearly constant, the same conclusion holds approximately.
\end{proof}

This analysis explains why a trained Reason RM can be more useful than direct rubric prompting. The RM is trained on same-problem reference-versus-perturbation groups, which makes its score discriminative in precisely the local comparisons used by GRPO. A saturated judge score collapses this extra signal.

\subsection{Finite-Sample Reward Noise}
\label{app:finite-sample-noise}

The reward used during GRPO is estimated from finite executor samples. The uplift estimator is unbiased, but its variance affects how reliably the policy can identify better traces inside a group.

Assume that the single execution score $X(P,R)$ and the $K$ samples used for $\widehat u_K(P,R)$ are independent. This independence assumption is made only for the variance calculation below; the expectation analysis above does not require it.

\begin{appproposition}[Variance bound for the sampled \method{} reward]
\label{prop:reward-variance}
Let
\begin{equation}
\widehat{\mathcal{R}}(P,R)
=
0.5\,X(P,R)+0.5\,\mrm(P,R)\,\widehat u_K(P,R).
\label{eq:app-sampled-reward}
\end{equation}
Under the independence assumptions of Proposition~\ref{prop:uplift-estimator},
\begin{equation}
\Var[\widehat{\mathcal{R}}(P,R)\mid P,R]
\le
\frac{1}{16}
+
\frac{\mrm(P,R)^2}{8K}
\le
\frac{1}{16}+\frac{1}{8K}.
\label{eq:app-sampled-reward-variance}
\end{equation}
\end{appproposition}

\begin{proof}
Using independence between $X(P,R)$ and $\widehat u_K(P,R)$,
\begin{equation}
\Var[\widehat{\mathcal{R}}(P,R)\mid P,R]
=
0.25\,\Var[X(P,R)]
+
0.25\,\mrm(P,R)^2\,\Var[\widehat u_K(P,R)].
\end{equation}
Since $X(P,R)$ is Bernoulli with mean $\qtr(P,R)$,
\begin{equation}
\Var[X(P,R)]
=
\qtr(P,R)(1-\qtr(P,R))
\le
\frac{1}{4}.
\end{equation}
By Proposition~\ref{prop:uplift-estimator},
\begin{equation}
\Var[\widehat u_K(P,R)]\le\frac{1}{2K}.
\end{equation}
Therefore,
\begin{equation}
\begin{aligned}
\Var[\widehat{\mathcal{R}}(P,R)\mid P,R]
&\le
0.25\cdot\frac{1}{4}
+
0.25\,\mrm(P,R)^2\cdot\frac{1}{2K}\\
&=
\frac{1}{16}
+
\frac{\mrm(P,R)^2}{8K}.
\end{aligned}
\end{equation}
Because $\mrm(P,R)\in[0,1]$,
\begin{equation}
\frac{1}{16}
+
\frac{\mrm(P,R)^2}{8K}
\le
\frac{1}{16}+\frac{1}{8K}.
\end{equation}
\end{proof}

The bound shows the role of the rollout count $K$: increasing $K$ reduces the variance of the uplift component, but it does not remove the variance of the single verifier outcome. Therefore, beyond a moderate $K$, the marginal benefit of additional executor comparisons can be limited, especially when training optimizes a stochastic temperature-$0.5$ executor utility while evaluation measures deterministic temperature-$0$ pass@1.

\begin{appcorollary}[Finite-sample sign reliability]
\label{cor:sign-reliability}
Consider two traces $R_a,R_b$ for the same problem with population reward gap
\begin{equation}
\Delta
=
\Expect[\widehat{\mathcal{R}}(P,R_a)-\widehat{\mathcal{R}}(P,R_b)\mid P,R_a,R_b]
>
0.
\label{eq:app-population-gap}
\end{equation}
Let the empirical gap be $\widehat{\Delta}$. If the two reward estimates are independent, then
\begin{equation}
\Prob[\widehat{\Delta}\le0]
\le
\frac{
\Var[\widehat{\mathcal{R}}(P,R_a)]
+
\Var[\widehat{\mathcal{R}}(P,R_b)]
}{\Delta^2}.
\label{eq:app-sign-reliability}
\end{equation}
\end{appcorollary}

\begin{proof}
Let
\begin{equation}
\varepsilon=\widehat{\Delta}-\Delta.
\end{equation}
Then $\Expect[\varepsilon]=0$. If $\widehat{\Delta}\le0$, then
\begin{equation}
\varepsilon\le-\Delta,
\end{equation}
which implies
\begin{equation}
|\varepsilon|\ge\Delta.
\end{equation}
Therefore, by Chebyshev's inequality,
\begin{equation}
\Prob[\widehat{\Delta}\le0]
\le
\Prob[|\varepsilon|\ge\Delta]
\le
\frac{\Var[\varepsilon]}{\Delta^2}.
\end{equation}
Under independence of the two reward estimates,
\begin{equation}
\Var[\varepsilon]
=
\Var[\widehat{\mathcal{R}}(P,R_a)]
+
\Var[\widehat{\mathcal{R}}(P,R_b)].
\end{equation}
Substituting gives the result.
\end{proof}

This corollary is not a convergence theorem for GRPO. It only states that larger reward margins and lower reward variance make pairwise credit assignment more reliable. The product term helps by increasing the margin between traces whose executor utility is supported by high reasoning quality and traces whose utility is weak, noisy, or unsupported.

\subsection{Robustness to RM Score Error}
\label{app:rm-error}

The theory above does not require the RM to be perfect. Since the RM score is bounded and multiplied by uplift, the effect of RM score error is also bounded.

Let $\mrm(P,R)$ be an ideal bounded score used for analysis and let $\widetilde m(P,R)$ be the score actually used in the reward. Suppose
\begin{equation}
|\widetilde m(P,R)-\mrm(P,R)|\le\eta
\label{eq:app-rm-error-assumption}
\end{equation}
for a given trace.

\begin{appproposition}[Bounded effect of RM score error]
\label{prop:rm-error-bound}
For any fixed $P,R$,
\begin{equation}
\left|
\Expect[\widetilde{\mathcal{R}}(P,R)-\mathcal{R}(P,R)\mid P,R]
\right|
\le
\frac{\eta}{2},
\label{eq:app-rm-error-bound}
\end{equation}
where $\widetilde{\mathcal{R}}$ denotes the reward computed with $\widetilde m$ and $\mathcal{R}$ denotes the reward computed with $\mrm$.
\end{appproposition}

\begin{proof}
The only changed term is the RM-uplift term:
\begin{equation}
\Expect[\widetilde{\mathcal{R}}-\mathcal{R}\mid P,R]
=
0.5\,(\widetilde m(P,R)-\mrm(P,R))\,\utr(P,R).
\end{equation}
Taking absolute values,
\begin{equation}
\left|
\Expect[\widetilde{\mathcal{R}}-\mathcal{R}\mid P,R]
\right|
=
0.5\,|\widetilde m(P,R)-\mrm(P,R)|\,|\utr(P,R)|.
\end{equation}
Since $|\widetilde m-\mrm|\le\eta$ and $|\utr(P,R)|\le1$,
\begin{equation}
\left|
\Expect[\widetilde{\mathcal{R}}-\mathcal{R}\mid P,R]
\right|
\le
0.5\,\eta.
\end{equation}
\end{proof}

For pairwise comparisons, if two traces each have RM score error at most $\eta$, then the error in the expected reward gap is at most $\eta$, because each trace contributes at most $\eta/2$. Therefore, a population reward margin larger than $\eta$ is stable to such bounded RM score perturbations. This motivates training a calibrated Reason RM: the smaller and less saturated the score error, the more reliably the reward preserves the intended within-group ordering.

\subsection{Scope of the Analysis}
\label{app:theory-scope}

The analysis establishes properties of the training reward under a fixed executor and matched training-time sampling. It does not prove that every improvement in the training-time objective must transfer to deterministic evaluation. Such a statement would require additional assumptions relating $\qtr(P,R)$ at temperature $0.5$ to $\qev(P,R)$ at temperature $0$, as well as assumptions about policy optimization and generalization. The main experiments provide this empirical link: the Reason RM and uplift estimator are removed at evaluation time, and improvements are measured only through the fixed planner-executor chain.

Within its stated scope, the analysis gives the following conclusions:
\begin{enumerate}
    \item Executor uplift is an unbiased estimate of the conditional effect of supplying trace $R$ to the frozen executor.
    \item The expected \method{} reward reduces, up to a problem-level constant, to a quality-weighted executor-uplift objective.
    \item The product term prevents high RM scores from creating trace-specific reward when the trace has zero executor utility.
    \item The verifier term provides a direct task-success anchor and prevents the reward from depending only on RM-weighted relative uplift.
    \item A saturated LLM-as-judge score collapses to an execution-only ranking and therefore cannot substitute for a calibrated Reason RM.
\end{enumerate}

\clearpage
\bibliographystyle{plainnat}
\bibliography{main}

\end{document}